\documentclass[dvipsnames]{article} %
\usepackage{colm2024_conference}

\usepackage{booktabs}
\usepackage{enumitem}
\usepackage{wrapfig}


\usepackage{algorithm}
\usepackage{algpseudocode}

\usepackage{graphicx}
\usepackage[misc]{ifsym}
\usepackage{mathtools}   

\usepackage{microtype}
\usepackage{colortbl}
\usepackage[utf8]{inputenc}
\usepackage[T1]{fontenc}
\definecolor{lightgray}{rgb}{0.9,0.9,0.9}
\usepackage{caption}
\usepackage{subcaption}
\usepackage{setspace}
\usepackage{url}
\usepackage{multirow}
\usepackage{tabularx}
\usepackage{blindtext}
\usepackage{pgfplots}
\pgfplotsset{compat=1.18} 
\usepackage{tikz}
\usetikzlibrary{er,positioning,bayesnet}
\usepackage{makecell}
\usepackage{tipa}
\let\tipaThickSpace\;
\let\tipaDoubleBar\|
\let\tipaNegThin\!
\DeclareRobustCommand{\;}{\ifmmode\mskip\thickmuskip\else\tipaThickSpace\fi}
\DeclareRobustCommand{\|}{\ifmmode\Vert\else\tipaDoubleBar\fi}
\DeclareRobustCommand{\!}{\ifmmode\mskip-\thinmuskip\else\tipaNegThin\fi}
\usepackage{siunitx}
\usepackage{nicefrac}
\usepackage{listings}
\usepackage[raster,skins, most]{tcolorbox} %
\usepackage{xltabular}
\usepackage{adjustbox}
\usepackage{xurl}
\usepackage{rotating}
\usepackage[normalem]{ulem}
\usepackage{xspace}

\usepackage{fontawesome}
\usepackage{fancyvrb}
\usepackage{color}
\definecolor{green_first}{RGB}{168, 209, 176}   
\definecolor{green_second}{RGB}{200, 235, 200}  
\definecolor{green_third}{RGB}{235, 255, 235}   
\definecolor{light_green_table}{RGB}{220, 255, 220}  
\definecolor{light_purple_table}{RGB}{235, 225, 255}  
\definecolor{light_green}{rgb}{0.569, 0.800, 0.459}
\definecolor{blue_dist}{rgb}{0.192,0.443,0.651}

\definecolor{purple_dist}{RGB}{125,94,237}

\definecolor{orange_dist}{rgb}{0.812,0.545,0.239}
\definecolor{yellow_dist}{rgb}{0.918,0.804,0.463}
\definecolor{website}{rgb}{0.9333333333333333, 0.10980392156862745, 0.592156862745098} 
\definecolor{pm_rowcolor}{rgb}{0.85, 0.90, 0.84} 
\definecolor{medium_gray}{RGB}{150, 150, 150}  
\definecolor{medium_purple}{RGB}{150, 120, 200}  
\newtcolorbox{promptbox}[2][Prompt]{
    colback=black!5!white,        
    arc=5pt,                      
    boxrule=0.5pt,                
    fonttitle=\bfseries,          
    title=#1,                     
    before upper={\small},        
    fontupper=\fontfamily{ptm}\selectfont, 
    colframe=#2,                  
    left=3pt,                     
    right=3pt,                    
    top=3pt,                      
    bottom=3pt,                   
    boxsep=3pt,                   
    toptitle=1pt,                 
    bottomtitle=1pt,              
    lefttitle=1pt,                
    righttitle=1pt,               
}

\usepackage{xcolor}      
\newcommand\inb[1]{\colorbox{gray!20}{\lstinline|#1|}}

\useunder{\uline}{\ul}{}

\usepackage{amsmath,amsfonts,amssymb,bm}
\usepackage{amsthm}
\newtheorem{theorem}{Theorem}
\newtheorem{assumption}{Assumption}

\theoremstyle{remark}
\newtheorem{remark}{Remark}

\def\eqref#1{equation~\ref{#1}}

\newcommand{\E}{\mathbb{E}}

\def\1{\bm{1}}

\DeclareMathAlphabet{\mathsfit}{\encodingdefault}{\sfdefault}{m}{sl}
\SetMathAlphabet{\mathsfit}{bold}{\encodingdefault}{\sfdefault}{bx}{n}

\usepackage{makecell}
\usetikzlibrary{tikzmark}
\makeatletter
\newcommand*\myfontsize{%
  \@setfontsize\myfontsize{7}{8}%
}
\makeatother

\definecolor{uclablue}{RGB}{159, 195, 224}

\definecolor{uclagold}{RGB}{255, 240, 180}

\definecolor{aliceblue}{RGB}{255, 238, 241}

\definecolor{cadmiumgreen}{rgb}{0.0, 0.42, 0.24}

\definecolor{myred}{rgb}{0.7, 0.3, 0.0}
\definecolor{myblue}{rgb}{0.2, 0.3, 0.6}
\definecolor{babygreen}{rgb}{0.85, 0.97, 0.85}

\definecolor{purple1}{RGB}{126, 107, 196}
\definecolor{purple2}{RGB}{199, 158, 207}
\definecolor{purple3}{RGB}{214, 200, 255}
\definecolor{purple4}{RGB}{254, 240, 255}

\definecolor{deepblue}{RGB}{48, 58, 82}

\newcommand{\symboletongyi}{\raisebox{0pt}{~\includegraphics[scale=0.012]{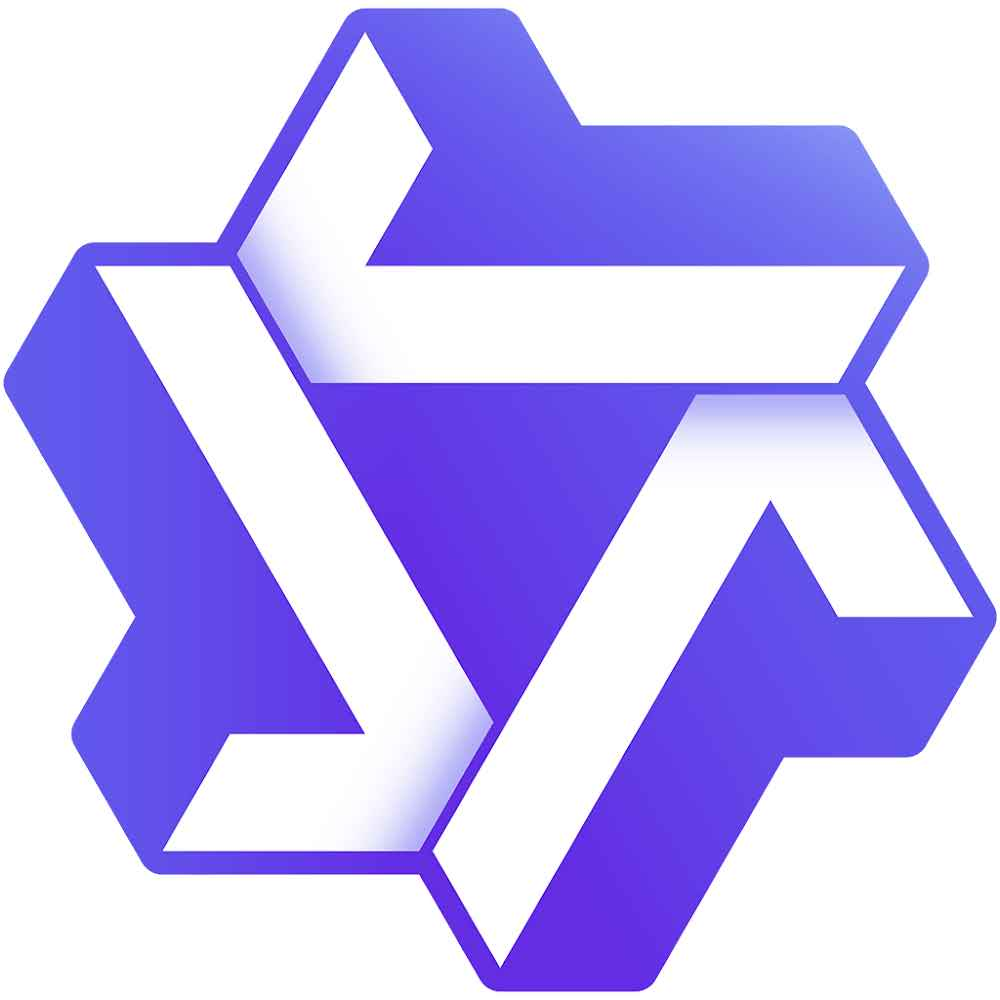}}~}

\definecolor{deepPurple}{HTML}{330066}

\definecolor{uclablue_old}{rgb}{0.15, 0.45, 0.68}
\hypersetup{
    breaklinks,
    citecolor=uclablue_old,
    colorlinks=true,
}

\newtcolorbox{mybox}[2][]
  {colback = black!5!white, colframe = black!75!black, fonttitle = \bfseries,
    colbacktitle = black!100!black, enhanced, before upper={\fontsize{8}{11}\obeyspaces\obeylines\selectfont}, fontupper=\selectfont,
    attach boxed title to top left={yshift=-2.2mm,xshift=4mm},
    title=#2,#1}

\title{%
\begin{tabular}[t]{l} 
  \parbox[t]{0.8\textwidth}{\centering 
    Scaling Self-Evolving Agents via Parametric Memory
  }
\end{tabular}
}

\author{
Tao Ren$^*$,
Weiyao Luo,
Hui Yang,
Rongzhi Zhu,
Xiang Huang,
Yuchuan Wu$^{ (\textrm{\Letter})}$,
Bingxue Chou,
Jieping Ye,
Jiafeng Liang$^{ (\textrm{\Letter})}$,
Yongbin Li$^{ (\textrm{\Letter})}$,
Yijie Peng$^{ (\textrm{\Letter})}$
  \\[1em]               
  {\fontsize{10pt}{11pt}\selectfont          
Qwen-Character Team\symboletongyi, Alibaba Group\\
Peking University}\\
}

\begin{document}

\maketitle

\begingroup
  \renewcommand\thefootnote{\Letter\space}  
  \footnotetext{Corresponding authors: jfliang@ir.hit.edu.cn, pengyijie@pku.edu.cn; * Email to: rtkenny@stu.pku.edu.cn} 
\endgroup

\vspace{-1em}

\begin{abstract}
Existing memory-augmented LLM agents store past experience exclusively in prompt space, as textual summaries or retrieved passages, while keeping model parameters frozen throughout a rollout. Such agents can \emph{look up} what they have seen but cannot \emph{learn from} it: their policy is unchanged by experience, and any information dropped from the context is permanently lost. We introduce \texttt{TMEM}, a self-evolving parametric memory framework in which the agent not only compresses history into explicit memory but also absorbs distilled supervision into fast LoRA weights $\Delta_t$ via lightweight online updates, genuinely altering its future behavior within a single episode. We formalize this as an agentic decision process with fast-weight rollout dynamics: actions are sampled from $\pi_{\theta_0+\Delta_t}$, while extraction actions produce supervision that updates $\Delta_t$ for subsequent decisions. This view makes the extraction policy directly optimizable by RL: training $\theta_0$ improves not only task actions but also the quality of the data used for online LoRA adaptation. We further propose SVD-based initialization of the LoRA subspace to accelerate online convergence. Experiments on LoCoMo, LongMemEval-S, multi-objective search, and CL-Bench show that \texttt{TMEM} consistently outperforms summary-based and retrieval-based baselines across different model scales.
\end{abstract}

\section{Introduction}

\vspace{0.3em}
\noindent\emph{``We are who we are because of what we learn and what we remember.''} \hfill\mbox{--- Eric Kandel}
\vspace{0.3em}

\noindent Long-horizon LLM agents are expected to operate over experiences that far exceed a single context window: multi-session conversations, extended web-search trajectories, iterative tool use, and evolving user preferences\citep{li2026scout,li2026optimallowrank,ren2025riskpo,yang2026omnimaskedgd}. In these settings, memory should do more than preserve a compressed record of past tokens; it should change how future actions are produced. A useful analogy comes from biological memory: learning and memory are coupled because experience is ultimately reflected in the same substrate that performs computation \citep{liang2025aimeetsbrain, liang2026perceptionwithoutengagement}. Current LLM agents largely break this coupling. They either keep more history in the prompt or store past interactions in external memories such as summaries and retrieval indices, while the underlying model parameters remain frozen throughout the rollout. As a result, experience influences future behavior only when it is explicitly brought back as input tokens.

The simplest way to preserve experience is to keep the full interaction history in context, but this strategy fails along two axes. Computationally, attention cost grows quickly with sequence length, making long episodes expensive\citep{yu2025memagent,ren2025flops,ren2025halforder,li2026adaptiverobustestimator}. Statistically, raw histories are dominated by redundant tool outputs, repeated dialogue, and task-irrelevant details, so the relevant evidence becomes sparse inside a noisy prompt. Even when the necessary information is technically present, a single forward pass must both locate it and use it correctly. Once the context is truncated or compressed, any omitted evidence no longer has a path to affect the policy.

Existing memory agents address this bottleneck mainly through prompt-space memory. Summary-based methods periodically compress the history into textual states, reducing context length but introducing a lossy bottleneck that can discard fine-grained facts. Retrieval-based methods preserve more verbatim evidence in an external index, but their effectiveness depends on embedding quality, query formulation, and index maintenance. These approaches are useful, yet they share a structural limitation: the learned policy itself does not change during the episode. The agent can consult stored experience, but it cannot internalize that experience into the computation used for later decisions. This raises a central question: can an LLM agent write useful experience into its own parameters at test time, so that memory shapes the policy rather than merely occupying the prompt?

We introduce \texttt{TMEM}, a self-evolving parametric memory framework for answering this question. During a rollout, the agent maintains a working context $h_t$, optional explicit memory $m_t$, and fast parametric memory $\Delta_t$ represented by LoRA weights. When the context budget is reached, the agent enters a memory-writing mode: it distills the current session into grounded QA-style supervision and applies a lightweight online SFT update to $\Delta_t$. Subsequent actions are then sampled from the adapted policy $\pi_{\theta_0+\Delta_t}$, where the base parameters $\theta_0$ remain fixed within the rollout. In this way, distilled experience can influence future reasoning through fast weights, without requiring the same evidence to be repeatedly reinserted into the prompt.

This formulation also changes what should be optimized. Memory extraction is not an auxiliary preprocessing step; it is an action whose quality determines the data used for later online adaptation. We therefore formalize agent execution as a fast-weight rollout in which ordinary task actions, memory-writing actions, explicit-memory updates, and LoRA updates are part of one decision process. During RL training, $\theta_0$ is optimized across rollouts while gradients are stopped through the online update operator. The resulting objective trains the base model not only to solve tasks, but also to produce supervision that makes its own future fast-weight updates useful.

Our contributions are as follows:
\begin{enumerate}
    
  \item We formalize \texttt{TMEM} as a fast-weight rollout process whose policy depends on working context, explicit memory, and online LoRA weights. Setting $\Delta_t \equiv 0$ recovers purely explicit-memory agents as special cases, unifying existing context-management designs under a single framework.

  \item We propose SVD-based initialization for the LoRA projection matrix, anchoring online updates to high-energy directions of the pretrained weights. This accelerates few-step adaptation and reduces per-trigger compute by fixing the projection matrix and updating only the coefficient matrix.
    
  \item We derive a stop-gradient policy optimization objective over fast-weight rollouts, allowing outcome rewards to improve both task behavior and the extraction of supervision used by online LoRA updates.
    
  \item We validate \texttt{TMEM} on LoCoMo, LongMemEval-S, multi-objective search, and CL-Bench, showing consistent gains over summary-based and retrieval-based memory baselines across Qwen3-4B and Qwen3-8B backbones.
\end{enumerate}

\section{Related Works}

\paragraph{Summarization-based context management.}
Summarization-based methods compress long interaction histories into concise textual memories that fit within a fixed context budget. Early work explores heuristic or recurrent-style compression, such as \texttt{RecurrentGPT}'s language-state simulation and \texttt{MemoryBank}'s forgetting-inspired memory update \citep{zhou2023recurrentgpt, zhong2024memorybank}. Later methods learn when and how to compress context: \texttt{MemAgent} performs chunk-wise memory overwriting for long-context conversational tasks without tool calls \citep{yu2025memagent}, while \texttt{MEM1} applies the same idea to search-agent tasks that invoke an external search tool \citep{zhou2025mem1}; \texttt{SUPO} and \texttt{Context-Folding} further optimize summary use for long-horizon agents \citep{lu2025supo, sun2025scaling}. Recent systems further improve extraction quality or scalability through proactive correction and adaptive stopping \citep{yang2026promem, wang2026infmem}. However, their memory still lives entirely in context space, whose compression capacity is limited: once fine-grained evidence is omitted from a summary, later reasoning cannot recover it.

\paragraph{Retrieval-augmented context management.}
Retrieval-based methods preserve history in external stores and retrieve relevant snippets when needed. Standard RAG first augments generation with dense retrieval over documents \citep{lewis2020rag}; agent systems then extend this idea to interactive memory, including \texttt{Generative Agents}' reflection streams and \texttt{MemGPT}'s managed context cache \citep{park2023generative, packer2023memgpt}. More structured memory systems organize stored experience with entity graphs, scalable long-term stores, RL-trained memory management, or agentic linking, such as \texttt{GraphRAG}, \texttt{Mem0}, \texttt{Memory-R1}, and \texttt{A-MEM} \citep{graphrag2024, chhikara2025mem0, yan2025memoryR1, xu2026mem}. Retrieval avoids aggressive summarization, but it depends heavily on retriever quality: irrelevant or missed passages directly affect the policy, and constructing a high-quality memory RAG store with embeddings, metadata, and update rules is often time-consuming.

\paragraph{Test-time training.}
Test-time training (TTT) adapts model parameters at inference time to handle distribution shifts or instance-specific information \citep{sun2020test}. Recent LLM-oriented work scales this idea to long contexts or language tasks: end-to-end TTT adapts to long-context inputs, \texttt{LaCT} improves chunk-level update efficiency, and methods such as \texttt{TLM} and \texttt{TT-SI} use unlabeled test data or self-generated data for on-the-fly improvement \citep{tandon2025endtoend, lact2025, hu2025tlm, acikgoz2025ttsi}. \texttt{Doc-to-LoRA} further internalizes documents into LoRA weights with a hypernetwork \citep{charakorn2026doctolora}. However, existing TTT methods are not optimized for agentic memory: they typically train on the whole context or document rather than selectively learning from agent-relevant experience, making updates heavy and outside the agent's decision process. We instead treat fast weights $\Delta_t$ as part of the rollout dynamics and update them only from distilled memory.

\paragraph{Self-evolving agents.}
Self-evolving agents progressively improve their behavior by reusing feedback, failures, or self-generated experience. Early systems mainly evolve through non-parametric experience accumulation: \texttt{Reflexion} stores verbal self-critiques, \texttt{SPRING} converts papers into game-playing strategies, and \texttt{Voyager} grows a code-based skill library across Minecraft episodes \citep{shinn2023reflexion, wu2023spring, wang2024voyager}. Later methods strengthen this loop with distilled experience or self-generated training data, such as \texttt{ExpeL}'s textual lessons and \texttt{SPIN}'s self-play fine-tuning \citep{zhao2024expel, chen2024spin}. Recent work further targets autonomous agent improvement: \texttt{UI-Voyager} learns GUI control from failed trajectories, while \texttt{LSE} trains models to refine their own test-time contexts with improvement-based rewards \citep{lin2026uivoyager, chen2026lse}. Unlike these methods, which mostly evolve across episodes, tasks, or outer-loop updates, \texttt{TMEM} performs \emph{intra-episode parametric self-evolution}: distilled experience is written into fast LoRA weights $\Delta_t$ during a single rollout, so the policy itself changes in real time rather than only consulting external memories or revised prompts.

\section{Agentic Decision Process with Parametric Memory}

We formalize \texttt{TMEM} as an agentic decision process in which the policy can change during a rollout through fast LoRA updates. Each episode starts from a task prompt $q \sim D$ and contains at most $T$ model-generation events. Let $\mathcal{V}$ denote the set of finite token vocabulary and $\mathcal{V}^{*}$ the set of finite token sequences. At generation event $t\in\{1,\ldots,T\}$, the agent maintains a working context $h_t \in \mathcal{V}^{*}$, an explicit textual memory $m_t \in \mathcal{V}^{*}$, and fast parametric memory $\Delta_t$ represented by LoRA weights. The generated output is sampled from the adaptive policy
\begin{equation}
a_t \sim \pi_{\theta_0 + \Delta_t}(\cdot \mid c_t),
\qquad c_t \in \{(q, h_t, m_t), (q, h_t, m_t, d)\},
\end{equation}
where $a_t$ denotes the model-generated action. Its semantics are determined by the conditioning context: under the ordinary context $(q,h_t,m_t)$ it is a task/tool action or final response; while under the extraction context $(q,h_t,m_t,d)$, in which $d$ is a memory-writing prompt, it is a memory-writing action such as a summary, QA pairs, distilled facts, or instruction--response examples. The base parameters $\theta_0$ are fixed within a rollout but optimized across RL training; $\Delta_t$ changes within the rollout and acts as fast-weight memory. Purely explicit-memory agents are recovered by setting $\Delta_t \equiv 0$.

\begin{figure}[h]
\centering
\begin{promptbox}[Memory-writing prompt $d$ for QA-pair extraction]{purple_dist}

\textbf{Task: Generate grounded SFT QA pairs from the current session.}
\\
\small
Given the problem to solve, previous conversation history. Now you should create high-quality supervised fine-tuning (SFT) QA pairs grounded on the history.

Requirements:
\begin{enumerate}[leftmargin=*,nosep]
    \item Generate QA pairs adaptively based on how much useful information is present in the session.
    \begin{itemize}[leftmargin=*,nosep]
        \item If the session contains rich, concrete facts, generate more QA pairs.
        \item If the session has limited useful evidence, generate fewer QA pairs.
        \item If there is no usable evidence, return an empty JSON array.
    \end{itemize}
    \item You can generate QA pairs that capture the lessons learned from the session to help improve future interactions, such as preferences, plans, events, and temporal details, rather than just factual questions.
    \item Each question must be answerable using explicit information from the session.
    \item Each answer must be concise, factual, and directly supported by the session.
    \item Cover diverse types when possible: who/what/when/where, preferences, plans, events, and temporal details.
    \item Avoid duplicate or near-duplicate QA pairs, and keep wording natural and clear.
\end{enumerate}

Return ONLY a JSON array. Each item must be:
\begin{Verbatim}[fontsize=\small]
{
  "instruction": "<question>",
  "output": "<answer>"
}
\end{Verbatim}

Output the generated SFT QA pairs in the specified JSON format. Do not include any explanations or additional text.
\end{promptbox}
\caption{Memory-writing prompt $d$ used when the working context exceeds the preset length. The prompt instructs the agent to extract grounded SFT QA pairs from the current session for online LoRA updates.}
\label{fig:memory-writing-prompt}
\end{figure}

\subsection{Fast-Weight Rollout Dynamics}

Let $L_{\max}$ be the context budget and let $\ell(\cdot)$ denote token length. For compactness, write $s_t=(q,h_t,m_t,\Delta_t)$ and let $\{(s_t,a_t)\}_{t=1}^{T}$ be the complete sequence of model-generation events in a rollout. Memory operations are triggered at indices $\{t_i\}_{i=1}^{N}$ where $\ell(h_{t_i})+\ell(m_{t_i})>L_{\max}$ before the final response is emitted. We set $t_0=0$ and $t_{N+1}=T$, with empty ranges ignored. Between two consecutive boundaries, for $i=1,\ldots,N+1$ and $t=t_{i-1}+1,\ldots,t_i-1$, the agent performs normal interaction:
\begin{align}
a_t &\sim \pi_{\theta_0 + \Delta_t}(\cdot \mid q, h_t, m_t), \\
o_t &\sim \mathbb{P}_{env}(\cdot \mid q, h_t, m_t, a_t), \\
h_{t+1} &= (h_t, a_t, o_t), \qquad m_{t+1}=m_t, \qquad \Delta_{t+1}=\Delta_t,
\end{align}
where $o_t$ denotes the environment feedback or observation returned after executing $a_t$, such as a tool result, user response, or task-state update, and $h_{t+1}=(h_t, a_t, o_t)$ denotes concatenation. At each trigger index $t_i$ for $1\leq i\leq N$, the agent enters an extraction mode by appending the memory-writing prompt $d$ in Figure~\ref{fig:memory-writing-prompt}, which asks it to distill the accumulated context into QA-pair supervision:
\begin{equation}
a_{t_i} \sim \pi_{\theta_0 + \Delta_{t_i}}(\cdot \mid q, h_{t_i}, m_{t_i}, d),
\end{equation}
where the appended prompt $d$ makes $a_{t_i}$ a memory-writing action rather than a normal environment action. 
The explicit memory and fast weights are then updated according to the memory strategy:
\begin{equation}
(m_{t_i+1}, \Delta_{t_i+1}, h_{t_i+1}) =
\begin{cases}
(\varnothing, \mathcal{T}(\Delta_{t_i}, a_{t_i}), \varnothing),
& \text{\texttt{TMEM}},\\
(a_{t_i}, \Delta_{t_i}, \varnothing),
& \text{summary-only baseline}.
\end{cases}
\end{equation}
Here $\mathcal{T}$ is a lightweight online SFT/LoRA update that writes the extracted supervision into fast weights. In \texttt{TMEM}, the extraction action $a_{t_i}$ is structured as QA-style supervision: it is absorbed into $\Delta_{t_i+1}$ through $\mathcal{T}$, while no textual memory is carried forward. In the summary-only baseline, the fast weights remain unchanged and the extracted summary $a_{t_i}$ is carried forward as explicit prompt-space memory. At the final boundary $t_{N+1}=T$, the action $a_T$ is the terminal task response sampled from the ordinary context $(q,h_T,m_T)$; it receives the episode reward and produces no subsequent environment transition.

\begin{figure}[t]
\centering
\includegraphics[width=0.95\linewidth]{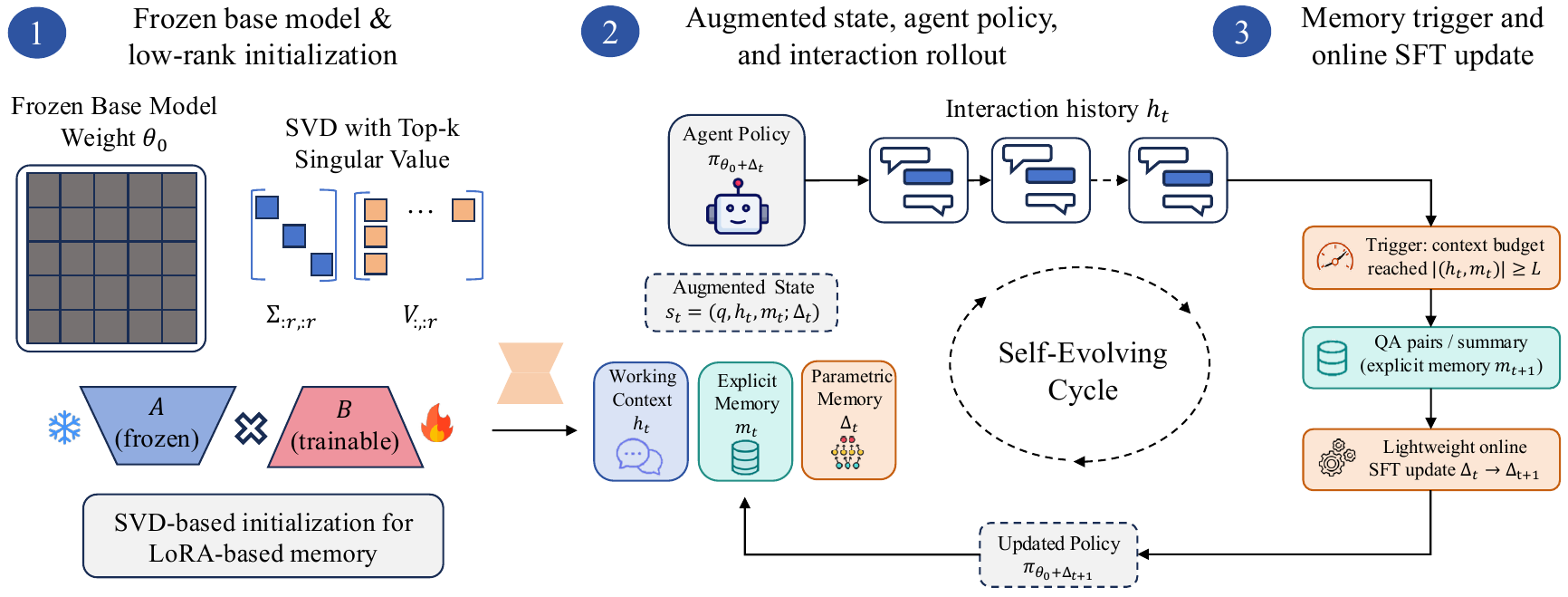}
\caption{Architecture of \texttt{TMEM}.}
\label{fig:ttt-memory-framework}
\end{figure}

\subsection{Parametric Memory Initialization via SVD}
Standard LoRA initializes $\Delta_0 = BA$ with $B = 0$ and $A$ drawn randomly, so that the adapted model starts identical to the base model $\theta_0$ while the low-rank subspace explored during training is arbitrary. In our setting, each TTT trigger has only a small number of gradient steps, so random subspaces waste early updates discovering where useful adaptation should occur. We instead initialize the projection matrix $A$ from the pretrained weight spectrum and train only the coefficient matrix $B$ online.

In \texttt{TMEM}, the LoRA adapters are applied only to the feed-forward network (FFN) projections in the Transformer backbone. For each selected target weight $W \in \mathbb{R}^{d_{\text{out}} \times d_{\text{in}}}$, let $W = U \Sigma V^{\top}$ and denote the top-$r$ right singular vectors and singular values by $V_r \in \mathbb{R}^{d_{\text{in}} \times r}$ and $\Sigma_r \in \mathbb{R}^{r \times r}$. We set
\begin{equation}
B_0 = 0, \qquad A_0 = \Sigma_r V_r^{\top};
\end{equation}
Because $B_0 = 0$, the initialization $\Delta_0 = B_0 A_0 = 0$ and the rollout begins from the original base policy. The difference from standard LoRA is that $A_0$ already spans the rank-$r$ subspace associated with the largest singular values of $W$. During online TTT, we keep $A=A_0$ fixed and update only $B$, so the few available gradient steps learn coefficients within a pretrained, high-energy subspace rather than jointly searching for both the subspace and the coefficients.


\begin{theorem}[Approximation advantage of SVD initialization]\label{thm:svd}
Under the setup and assumptions in Appendix~\ref{app:svd}, SVD initialization achieves no larger approximation error than random Gaussian initialization:
\[
\mathcal{E}(A_{\textup{SVD}}) \;\leq\; \E\bigl[\mathcal{E}(A_{\textup{rand}})\bigr],
\]
and the inequality is strict whenever $\rho_r(\Delta^{*}, W) > r / d_{\textup{in}}$.
\end{theorem}

This theorem justifies using SVD to initialize the LoRA subspace: before any online update, it already provides a better (or equal) low-rank approximation target than a random subspace in expectation.

\section{Policy Optimization with Fast-Weight Rollouts}
The fast-weight rollout couples two kinds of actions sampled from the same policy: ordinary actions that interact with the environment and extraction actions, triggered by the memory prompt $d$, that become training data for later LoRA updates. We optimize the base model parameters $\theta_0$ so that both kinds of actions improve the final task reward. During each sampled rollout, $\theta_0$ is held fixed while $\Delta_t$ evolves through $\mathcal{T}$; across RL updates, $\theta_0$ is trained to become a better initialization for this self-evolving process.

Let $\tau$ denote the full rollout, including normal actions, observations, extraction actions, explicit-memory updates, and fast-weight updates. We use an outcome reward $R(\tau)$ and optimize
\begin{equation}
J(\theta_0)=\mathbb{E}_{\tau\sim p_{\theta_0,\mathcal{T},\mathbb{P}_{env}}}\left[R(\tau)\right].
\end{equation}
In our formulation, the LoRA fast weights $\Delta_t$ are test-time memory states produced by the online operator $\mathcal{T}$, rather than RL parameters to optimize directly. During RL, we therefore treat $\mathcal{T}$ as part of the rollout transition and apply stop-gradient through the update path that produces $\Delta_t$. This design makes the optimization target explicit: train the base parameters $\theta_0$ to produce better ordinary actions and better QA-style extraction actions, so that the induced test-time updates are more useful for later decisions. Using a segment-wise likelihood-ratio decomposition over the fast-weight rollout, define the boundary context
\[
\bar c_i =
\begin{cases}
(q,h_{t_i},m_{t_i},d), & 1\leq i\leq N,\\
(q,h_T,m_T), & i=N+1.
\end{cases}
\]
The corresponding policy gradient estimation is
\begin{equation}
\nabla_{\theta_0}J(\theta_0)
\approx
\mathbb{E}_{\tau}\left[
R(\tau)
\left(
\sum_{i=1}^{N+1}
\left[
\sum_{t=t_{i-1}+1}^{t_i-1}
\nabla_{\theta_0}\log\pi_{\theta_0+\operatorname{sg}(\Delta_t)}(a_t\mid q,h_t,m_t)
+
\nabla_{\theta_0}\log\pi_{\theta_0+\operatorname{sg}(\Delta_{t_i})}(a_{t_i}\mid \bar c_i)
\right]
\right)
\right],
\end{equation}
where $\operatorname{sg}(\cdot)$ denotes stop-gradient and $d$ is the extraction prompt used to elicit QA-pair supervision. For each segment, the inner sum covers ordinary actions before the next boundary, and the boundary log-probability trains either the extraction action at $t_i$ when $i\leq N$ or the final response when $i=N+1$. A short derivation is provided in Appendix~\ref{app:pg}.

Gradients from normal actions improve reasoning and tool-use behavior, while gradients from extraction actions improve the model's ability to produce supervision that can be absorbed by $\mathcal{T}$. Since later rewards are generated under adapted policies $\pi_{\theta_0+\Delta_t}$, the RL signal favors base models that are not only capable at inference time but also easy to specialize through a small number of online LoRA updates.

\section{Experiments}

\subsection{Experimental Setup}

We evaluate four memory strategies---no memory, summary-based memory, retrieval-based memory, and \texttt{TMEM}---with Qwen3-4B and Qwen3-8B backbones across four task families: LoCoMo, LongMemEval-S, multi-objective search, and CL-Bench. Within each benchmark, all memory strategies use the same backbone model, while the remaining task-specific settings follow the original benchmark protocols. We report three-run evaluations under fixed settings and use the run average as the main score.

\paragraph{Baselines.}
The no-memory baseline keeps only the current working context~\citep{yao2022react}. The summary-based baseline uses prompt-space memory rewriting and is instantiated differently by task family: for conversational and context-learning tasks (LoCoMo, LongMemEval-S, and CL-Bench), which require no external tool calls, we follow \texttt{MemAgent}~\citep{yu2025memagent}, which performs chunk-wise memory overwriting over long dialogue and document contexts; for search-agent tasks (multi-objective search), which require an active search tool, we follow \texttt{MEM1}~\citep{zhou2025mem1}, which maintains a running compressed memory across search steps. The retrieval-based baseline uses \texttt{A-MEM}-style external memory retrieval~\citep{xu2026mem}. These baselines do not update model parameters during an evaluation episode. For all methods, the input examples, answer templates, and scoring rules are shared. When reporting efficiency, we include the cost of memory-specific operations, including retrieval/index access for retrieval memory, summary rewriting for summary memory, and online LoRA updates for \texttt{TMEM}.

\paragraph{Metrics.}
For QA-style benchmarks(LoCoMo and LongMemEval), we report token-level F1 and Exact Match (EM) after lowercasing, stripping punctuation, and normalizing whitespace. Scores are computed per instance and then averaged over the evaluation set. For multi-objective search, the final answer is parsed as an ordered answer list; each position is matched to the corresponding objective, and episode-level F1/EM is averaged over positions before aggregating across episodes. For CL-Bench, we use Qwen3-Max as an LLM-as-judge to score each response against the provided rubric, and we report rubric-based accuracy on the filtered evaluation set; category columns are subset accuracies, while the total column is instance-level accuracy over all 289 retained examples. Unless otherwise noted, each reported table number is the mean over 3 independent runs, and the $\pm$ term denotes the standard deviation across runs.

\paragraph{Conversational memory.}
We use LoCoMo~\citep{locomo2024} and LongMemEval-S~\citep{longmemeval2024} to evaluate long-term conversational memory. For LoCoMo, we use the official LoCoMo-10 QA setting, where each example provides a very long multi-session dialogue history and a memory question. For LongMemEval-S, we use the released 500-instance evaluation set under the full-haystack protocol: the model receives the full timestamped chat history and the memory question, not oracle evidence sessions.

\paragraph{Search-agent memory.}
We evaluate search-oriented memory with the multi-objective search protocol used by MEM1~\citep{zhou2025mem1}. Here an objective is an independent QA subgoal rather than an optimization objective: each episode combines 4 or 8 QA questions into one long-horizon search task. The agent searches for unresolved subquestions over multiple steps under the source task protocol and must finally output all answers in the original order using a fixed answer-list format. We compute F1 and EM on the ordered answer list.

\paragraph{Context learning.}
We evaluate context learning on CL-Bench~\citep{clbench2025} after filtering the released benchmark. The motivation for filtering is discriminability: when we run the full benchmark on 4B and 8B models, near-floor accuracy makes it difficult to distinguish memory strategies from each other or from the no-memory baseline. To construct a subset with sufficient signal, we first ran Qwen3-30A3B on every instance and scored each response against the provided rubric conditions; we then retained the instances that Qwen3-30A3B answered correctly at least partially, yielding a set where task difficulty is matched to the capability range of the target models rather than too hard for any of them. This filtering is motivated purely by discriminability: we do not select instances on which any 4B or 8B method performs well. The resulting set contains 289 instances across Domain Knowledge Reasoning, Procedural Task Execution, Rule System Application, and Empirical Discovery \& Simulation. Each instance provides a system prompt, task context, user task, and rubric; we then use Qwen3-Max as an LLM-as-judge to evaluate whether the response satisfies the rubric conditions, and an answer is counted as correct only if it passes the rubric judgment. We use this filtered CL-Bench set as a discriminative context-learning testbed.

\paragraph{TMEM implementation.}
Across all tasks, the method uses the same online LoRA module construction and TTT training hyperparameters. We set the LoRA rank to $r=6$ and attach LoRA adapters only to the feed-forward network (FFN) projection matrices (\texttt{gate\_proj}, \texttt{up\_proj}, and \texttt{down\_proj}) in the last 4 Transformer layers; attention projections are left unchanged. The SVD-initialized projection matrix $A=A_0$ is frozen, and only the coefficient matrix $B$ is updated online. Fast-weight updates are cumulative across triggers within an episode: each trigger starts from the current $B$ rather than resetting to zero. At each memory trigger, the extracted JSON QA supervision is parsed into instruction--answer pairs and used for online SFT with SGD, learning rate $5\times10^{-4}$, 5 epochs, and batch size 16. The trigger counter uses the token length of the current working context plus explicit memory; after a trigger, the working context is cleared, and in \texttt{TMEM} the extracted content is retained through $\Delta_t$ rather than as prompt text. We set the trigger budget to $L_{\max}=4096$ for LoCoMo, $L_{\max}=12288$ for LongMemEval-S, $L_{\max}=8192$ for the search-agent tasks, and $L_{\max}=4096$ for CL-Bench.

\paragraph{RL data construction and training configuration.}
For the RL phase, we directly use the existing RL data and task protocols from prior memory-agent work. Conversational RL follows the long-context QA data used by MemAgent~\citep{yu2025memagent}, while search-agent RL follows the multi-objective task data used by MEM1~\citep{zhou2025mem1}. Training uses the stop-gradient outcome-reward objective described in Section~4: final task reward trains both ordinary response tokens and memory-writing tokens, while gradients are not back-propagated through the online LoRA optimization. We keep the same RL data sources, task formats, train/evaluation separation, rollout budget, and update schedule across memory strategies so that post-RL comparisons differ only in the memory mechanism being optimized. We train with GRPO for 200 update steps using batch size 64, mini-batch size 16, learning rate $10^{-6}$, maximum response length 1024 tokens, and $n=8$ rollouts per prompt. For search-agent tasks, each episode allows up to 10 search turns.

\subsection{Evaluation before the RL}

We first evaluate each memory strategy before applying the RL phase. This isolates the effect of the memory mechanism itself: no memory, summary-based memory, retrieval-based memory, and \texttt{TMEM}.

\begin{table}[h]
\centering
\caption{F1 and Exact Match (EM) on conversational memory benchmarks (LoCoMo and LongMemEval-S). \textbf{Bold} indicates the best result per column.}
\label{tab:conv-results}
\small
\begin{tabular}{llcccc}
\toprule
 & & \multicolumn{2}{c}{LoCoMo} & \multicolumn{2}{c}{LongMemEval-S} \\
\cmidrule(lr){3-4} \cmidrule(lr){5-6}
Method & Model & F1 & EM & F1 & EM \\
\midrule
\multirow{2}{*}{No Memory}
  & Qwen3-4B & 23.33 $\pm$ 0.41 & 8.30 $\pm$ 0.22  & 5.30 $\pm$ 0.19  & 0.00 $\pm$ 0.00 \\
  & Qwen3-8B & 18.48 $\pm$ 0.39 & 3.12 $\pm$ 0.18  & 3.50 $\pm$ 0.16  & 0.00 $\pm$ 0.00 \\
\midrule
\multirow{2}{*}{MemAgent/MEM1}
  & Qwen3-4B & 20.74 $\pm$ 0.36 & 11.48 $\pm$ 0.29 & 36.45 $\pm$ 0.44 & 23.20 $\pm$ 0.38 \\
  & Qwen3-8B & 24.26 $\pm$ 0.35 & 15.90 $\pm$ 0.31 & 31.66 $\pm$ 0.42 & 17.40 $\pm$ 0.35 \\
\midrule
\multirow{2}{*}{AMEM}
  & Qwen3-4B & 25.66 $\pm$ 0.34 & 11.33 $\pm$ 0.28 & 29.22 $\pm$ 0.41 & 17.76 $\pm$ 0.36 \\
  & Qwen3-8B & 20.92 $\pm$ 0.33 & 8.64 $\pm$ 0.26  & 28.99 $\pm$ 0.40 & 18.40 $\pm$ 0.34 \\
\midrule
\multirow{2}{*}{\texttt{TMEM}}
  & Qwen3-4B & 25.72 $\pm$ 0.32 & 15.40 $\pm$ 0.27 & 41.24 $\pm$ 0.45 & \textbf{25.54 $\pm$ 0.39} \\
  & Qwen3-8B & \textbf{26.75 $\pm$ 0.31} & \textbf{20.24 $\pm$ 0.29} & \textbf{41.87 $\pm$ 0.43} & 25.42 $\pm$ 0.37 \\
\bottomrule
\end{tabular}
\end{table}

\begin{table}[h]
\centering
\caption{F1 and Exact Match (EM) on multi-objective search memory benchmarks (4-objective and 8-objective). \textbf{Bold} indicates the best result per column.}
\label{tab:search-results}
\small
\begin{tabular}{llcccc}
\toprule
 & & \multicolumn{2}{c}{4-objective} & \multicolumn{2}{c}{8-objective} \\
\cmidrule(lr){3-4} \cmidrule(lr){5-6}
Method & Model & F1 & EM & F1 & EM \\
\midrule
\multirow{2}{*}{No Memory}
  & Qwen3-4B & 18.26 $\pm$ 0.37 & 9.35 $\pm$ 0.28 & 18.22 $\pm$ 0.35 & 9.28 $\pm$ 0.27 \\
  & Qwen3-8B & 19.30 $\pm$ 0.34 & 10.23 $\pm$ 0.25 & 17.89 $\pm$ 0.32 & 9.12 $\pm$ 0.24 \\
\midrule
\multirow{2}{*}{MemAgent/MEM1}
  & Qwen3-4B & 24.35 $\pm$ 0.33 & 14.72 $\pm$ 0.30 & 22.18 $\pm$ 0.31 & 13.36 $\pm$ 0.28 \\
  & Qwen3-8B & 25.23 $\pm$ 0.31 & 15.11 $\pm$ 0.29 & 22.87 $\pm$ 0.30 & 13.42 $\pm$ 0.27 \\
\midrule
\multirow{2}{*}{AMEM}
  & Qwen3-4B & 24.82 $\pm$ 0.32 & 15.13 $\pm$ 0.29 & 22.65 $\pm$ 0.30 & 13.87 $\pm$ 0.27 \\
  & Qwen3-8B & 25.36 $\pm$ 0.30 & 16.22 $\pm$ 0.28 & 23.45 $\pm$ 0.29 & 14.12 $\pm$ 0.26 \\
\midrule
\multirow{2}{*}{\texttt{TMEM}}
  & Qwen3-4B & \textbf{26.74 $\pm$ 0.28} & 16.26 $\pm$ 0.25 & 24.51 $\pm$ 0.27 & \textbf{15.03 $\pm$ 0.24} \\
  & Qwen3-8B & 26.36 $\pm$ 0.27 & \textbf{16.45 $\pm$ 0.24} & \textbf{25.12 $\pm$ 0.26} & \textbf{15.03 $\pm$ 0.23} \\
\bottomrule
\end{tabular}
\end{table}

\paragraph{Conversational and search memory results.}
Table~\ref{tab:conv-results} reports conversational F1 and EM on LoCoMo and LongMemEval-S, while Table~\ref{tab:search-results} reports search F1 and EM on the 4-objective and 8-objective splits. On conversational memory, \texttt{TMEM} obtains the strongest overall results. The LongMemEval-S gains are especially clear: over the best explicit-memory baseline, \texttt{TMEM} improves by +4.79 F1 / +2.34 EM with Qwen3-4B and +10.21 F1 / +7.02 EM with Qwen3-8B. On LoCoMo, the F1 gap for Qwen3-4B is small (25.72 vs. 25.66), but the EM gain is larger (+3.92 over the best baseline), and Qwen3-8B shows larger improvements on both F1 and EM. On search memory, \texttt{TMEM} improves over the strongest explicit-memory baseline by +1.92 F1 / +1.13 EM on the 4-objective split and +1.86 F1 / +1.16 EM on the 8-objective split for Qwen3-4B. For Qwen3-8B, the gains remain positive but smaller on the 4-objective split (+1.00 F1 / +0.23 EM), so we treat that case as directional rather than decisive.

\begin{table}[h]
\centering
\caption{Accuracy (\%) on the filtered CL-Bench evaluation set across four knowledge categories: \textbf{DK} (Domain Knowledge), \textbf{ED} (Empirical Discovery), \textbf{PT} (Procedural Task), \textbf{RS} (Rule System). \textbf{Bold} indicates the best result per column.}
\label{tab:clbench}
\small
\begin{tabular}{llccccc}
\toprule
Method & Model & DK & ED & PT & RS & Total \\
\midrule
\multirow{2}{*}{No Memory}
  & Qwen3-4B & \textbf{34.96 $\pm$ 0.72} & 10.00 $\pm$ 0.58 & 25.71 $\pm$ 0.81 & \textbf{36.36 $\pm$ 0.76} & 30.45 $\pm$ 0.63 \\
  & Qwen3-8B & \textbf{34.96 $\pm$ 0.69} & 6.67 $\pm$ 0.54  & 35.71 $\pm$ 0.79 & 34.85 $\pm$ 0.73 & 32.18 $\pm$ 0.61 \\
\midrule
\multirow{2}{*}{MemAgent}
  & Qwen3-4B & 26.02 $\pm$ 0.74 & 10.00 $\pm$ 0.56 & 30.00 $\pm$ 0.77 & 33.33 $\pm$ 0.74 & 26.99 $\pm$ 0.60 \\
  & Qwen3-8B & 33.33 $\pm$ 0.70 & 6.67 $\pm$ 0.53  & 38.57 $\pm$ 0.75 & 34.85 $\pm$ 0.72 & 32.18 $\pm$ 0.59 \\
\midrule
\multirow{2}{*}{AMEM}
  & Qwen3-4B & 28.46 $\pm$ 0.71 & 10.00 $\pm$ 0.55 & 25.71 $\pm$ 0.79 & 33.33 $\pm$ 0.73 & 26.99 $\pm$ 0.58 \\
  & Qwen3-8B & 31.71 $\pm$ 0.68 & 10.00 $\pm$ 0.56 & 30.00 $\pm$ 0.76 & 30.30 $\pm$ 0.71 & 28.72 $\pm$ 0.57 \\
\midrule
\multirow{2}{*}{\texttt{TMEM}}
  & Qwen3-4B & \textbf{34.96 $\pm$ 0.67} & \textbf{16.67 $\pm$ 0.62} & 34.29 $\pm$ 0.74 & 34.85 $\pm$ 0.70 & 32.87 $\pm$ 0.56 \\
  & Qwen3-8B & 34.15 $\pm$ 0.66 & 10.00 $\pm$ 0.55 & \textbf{42.86 $\pm$ 0.71} & 34.85 $\pm$ 0.69 & \textbf{33.91 $\pm$ 0.55} \\
\bottomrule
\end{tabular}
\end{table}

\paragraph{CL-Bench results.}
Table~\ref{tab:clbench} reports rubric accuracy on the filtered CL-Bench evaluation set. \texttt{TMEM} remains strongest overall: it achieves 32.87\% for Qwen3-4B and 33.91\% for Qwen3-8B, improving over the best explicit-memory baseline by +2.42 and +1.73 points, respectively. The category-level pattern is still mixed: \texttt{TMEM} has the largest gain on Empirical Discovery for Qwen3-4B and the best Procedural Task score for Qwen3-8B, while the no-memory baseline remains competitive on Rule System. We therefore treat these numbers as evidence on the filtered answerable set rather than as a claim about the full CL-Bench distribution. We also evaluate on the full unfiltered CL-Bench split, and detailed results are reported in Appendix Table~\ref{tab:clbench-unfiltered}.

\subsection{Evaluation after the RL}

\begin{figure}[h]
  \centering
  \begin{subfigure}[b]{0.48\linewidth}
    \includegraphics[width=\linewidth]{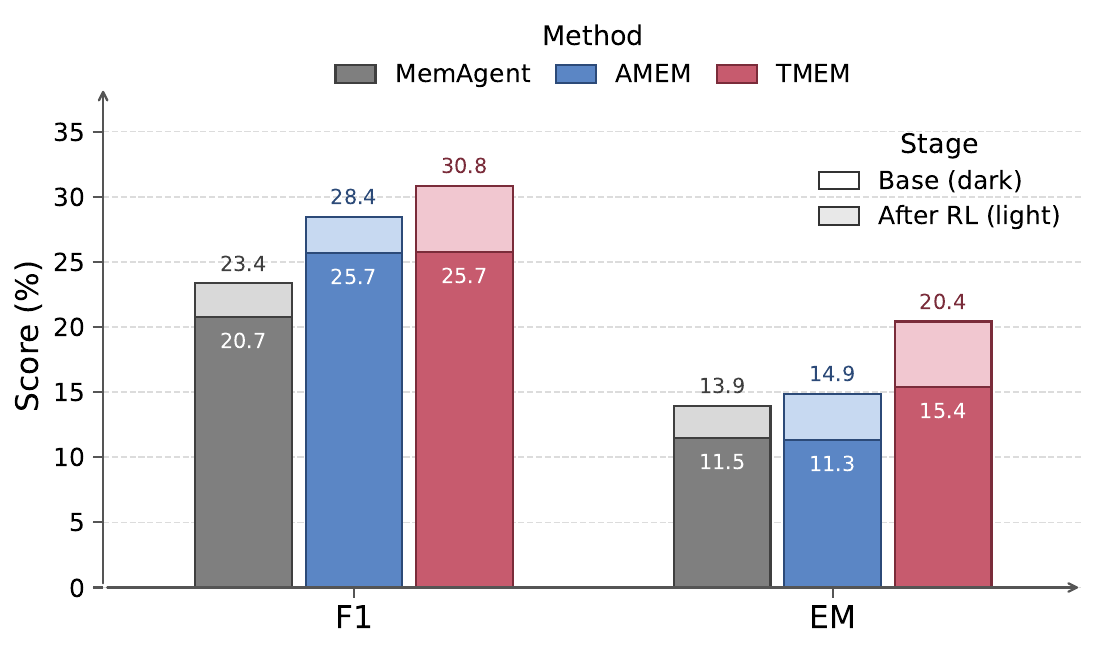}
    \caption{LoCoMo}
  \end{subfigure}
  \hfill
  \begin{subfigure}[b]{0.48\linewidth}
    \includegraphics[width=\linewidth]{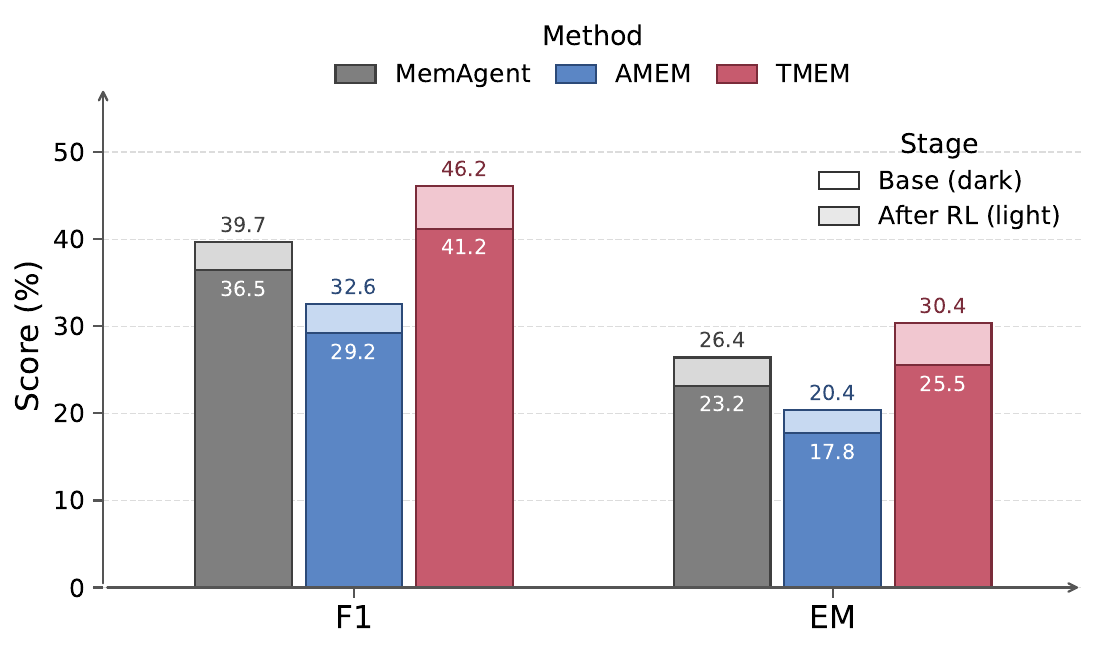}
    \caption{LongMemEval-S}
  \end{subfigure}
  \caption{F1 and Exact Match (EM) improvements after the RL phase on LoCoMo and LongMemEval-S. Stacked bars show the base score (dark) and RL gain (light) for each method.}
  \label{fig:rl-conv-search}
\end{figure}

\paragraph{Conversational RL gains.}
Figure~\ref{fig:rl-conv-search} summarizes the post-RL uplift on LoCoMo and LongMemEval-S. RL improves both F1 and EM for all three memory methods, but \texttt{TMEM} receives the largest absolute gains. On LoCoMo, \texttt{TMEM} improves by +5.12 F1 and +5.02 EM, exceeding MemAgent / MEM1 (+2.62 F1, +2.46 EM) and AMEM (+2.79 F1, +3.54 EM). On LongMemEval-S, \texttt{TMEM} again obtains the strongest gains (+4.92 F1, +4.84 EM), compared with +3.26 / +3.23 for MemAgent / MEM1 and +3.34 / +2.66 for AMEM. This pattern suggests that RL benefits are larger when the memory mechanism can adapt fast weights during the episode; all gains are computed from the corresponding three-run averaged scores.

\begin{figure}[h]
  \centering
  \begin{subfigure}[b]{0.48\linewidth}
    \includegraphics[width=\linewidth]{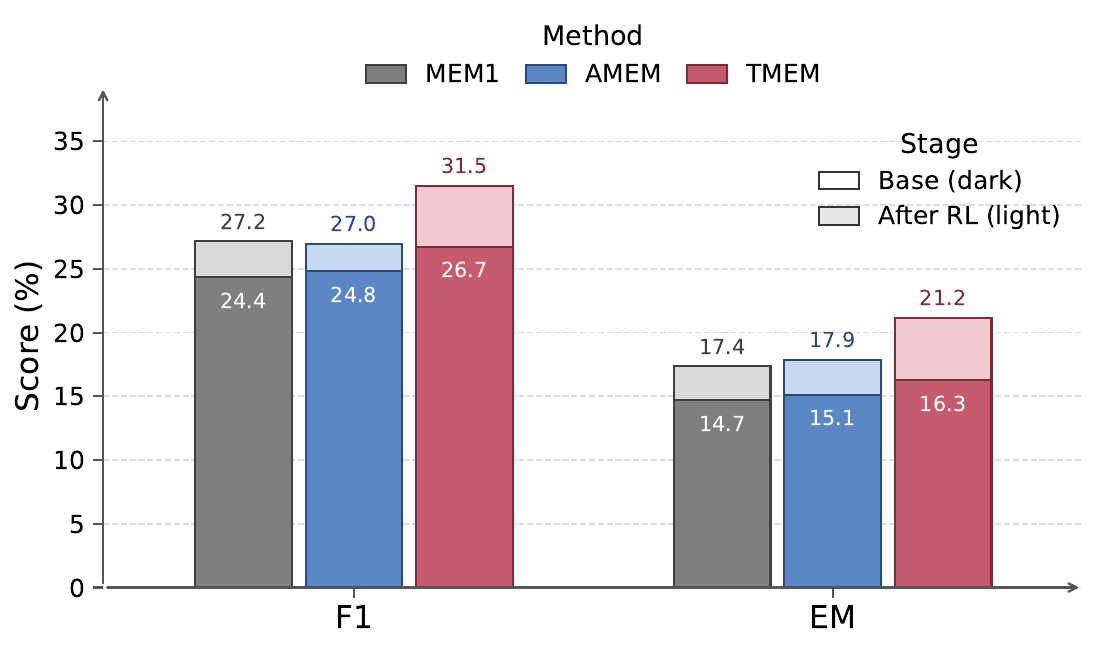}
    \caption{4-objective}
  \end{subfigure}
  \hfill
  \begin{subfigure}[b]{0.48\linewidth}
    \includegraphics[width=\linewidth]{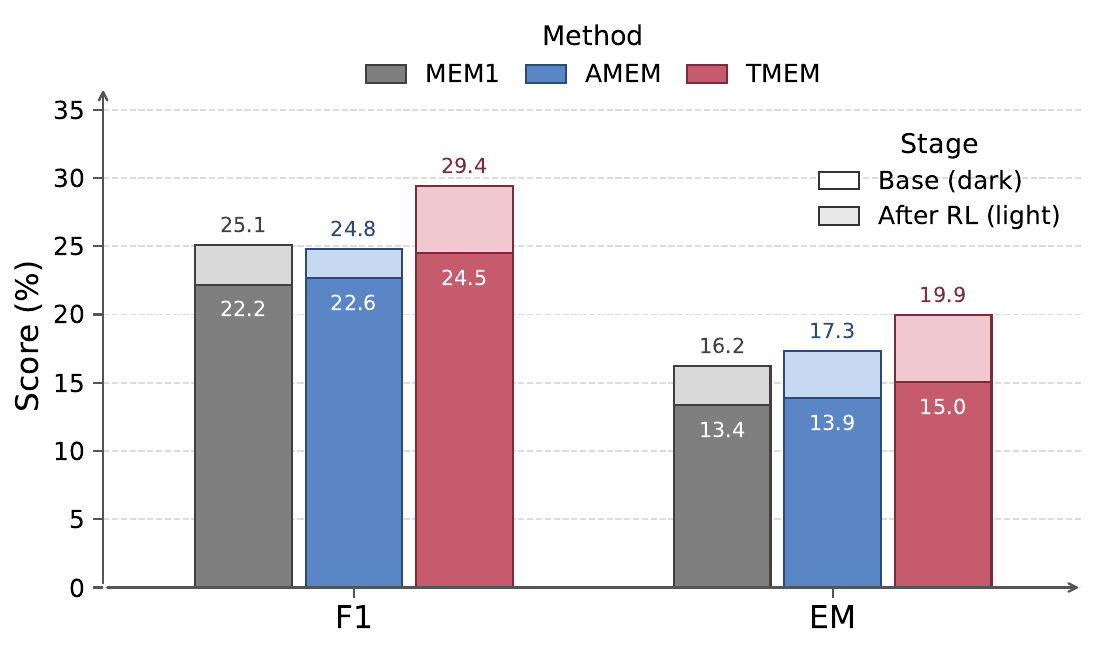}
    \caption{8-objective}
  \end{subfigure}
  \caption{F1 and Exact Match (EM) improvements after the RL phase on search-agent benchmarks. Stacked bars show the base score (dark) and RL gain (light) for each method.}
  \label{fig:rl-search}
\end{figure}

\paragraph{Search-agent RL gains.}
Figure~\ref{fig:rl-search} shows the same ordering on search-agent tasks. After RL, \texttt{TMEM} reaches 31.52 F1 / 21.17 EM on the 4-objective split and 29.43 F1 / 19.95 EM on the 8-objective split. These scores improve over its pre-RL version by +4.78 / +4.91 and +4.92 / +4.92 in F1 / EM, respectively, and exceed the strongest explicit-memory baseline after RL by +4.34 / +3.29 on the 4-objective split and +4.36 / +2.63 on the 8-objective split. The gains are consistent with the conversational results: optimizing the memory-writing behavior is helpful when the agent must preserve evidence for several independent subgoals.

\subsection{Ablation}

We ablate four design choices that are central to the fast-weight rollout design using Qwen3-4B. In each experiment, we vary one factor while keeping all others at their default values. These ablations focus on the quantitative controls for which we have complete results: trigger budget, SVD initialization, supervision form, and whether the SVD subspace is frozen.

\paragraph{Effect of the context budget $L_{\max}$.}
Figure~\ref{fig:context-budget-ablation} reports sensitivity to the context budget $L_{\max}$ on LoCoMo and LongMemEval-S. We vary $L_{\max}$ while keeping rank, learning rate, TTT epochs, and extraction policy fixed.

\begin{figure}[h]
\centering
\begin{subfigure}[t]{0.48\linewidth}
\centering
\includegraphics[width=\linewidth]{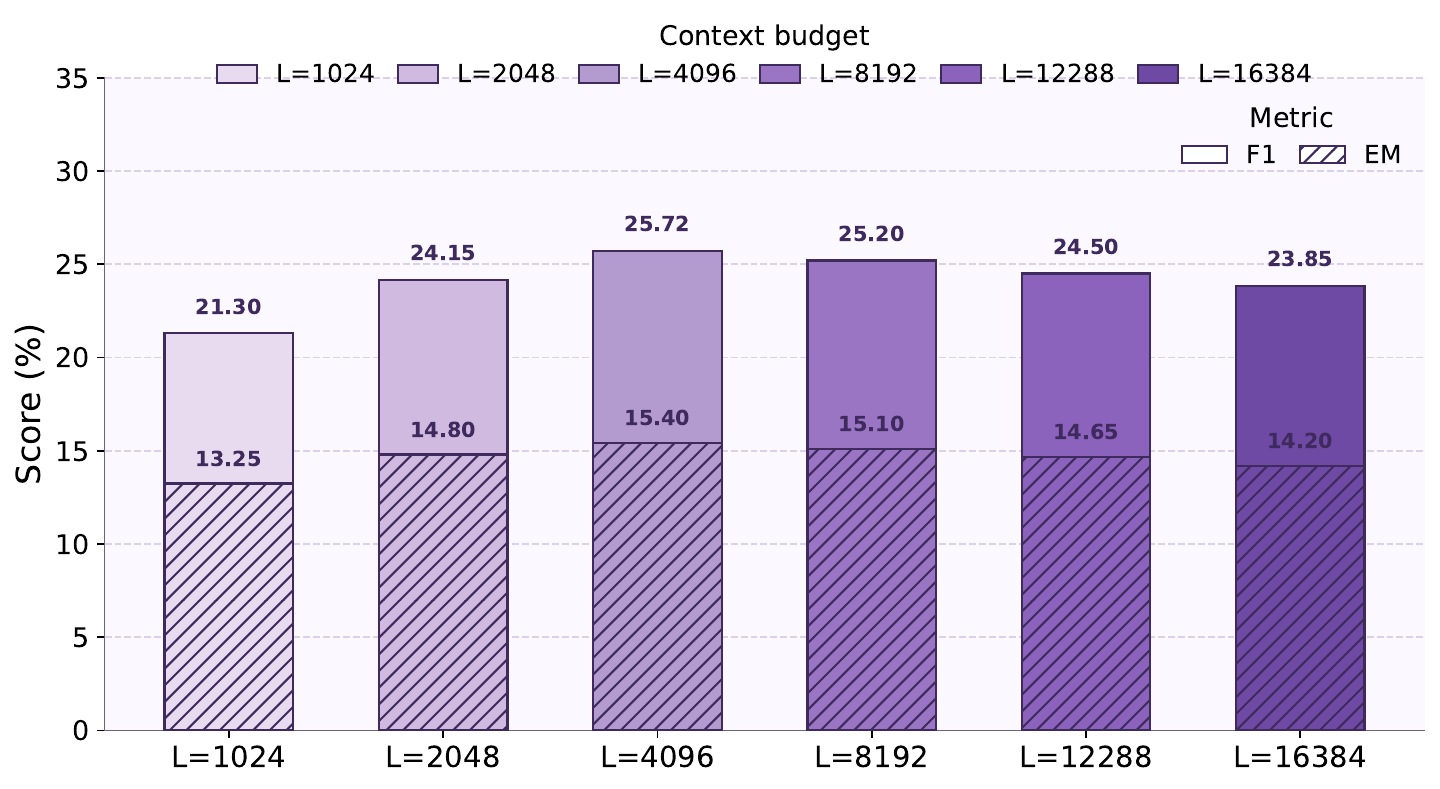}
\caption{LoCoMo}
\label{fig:context-budget-locomo}
\end{subfigure}
\hfill
\begin{subfigure}[t]{0.48\linewidth}
\centering
\includegraphics[width=\linewidth]{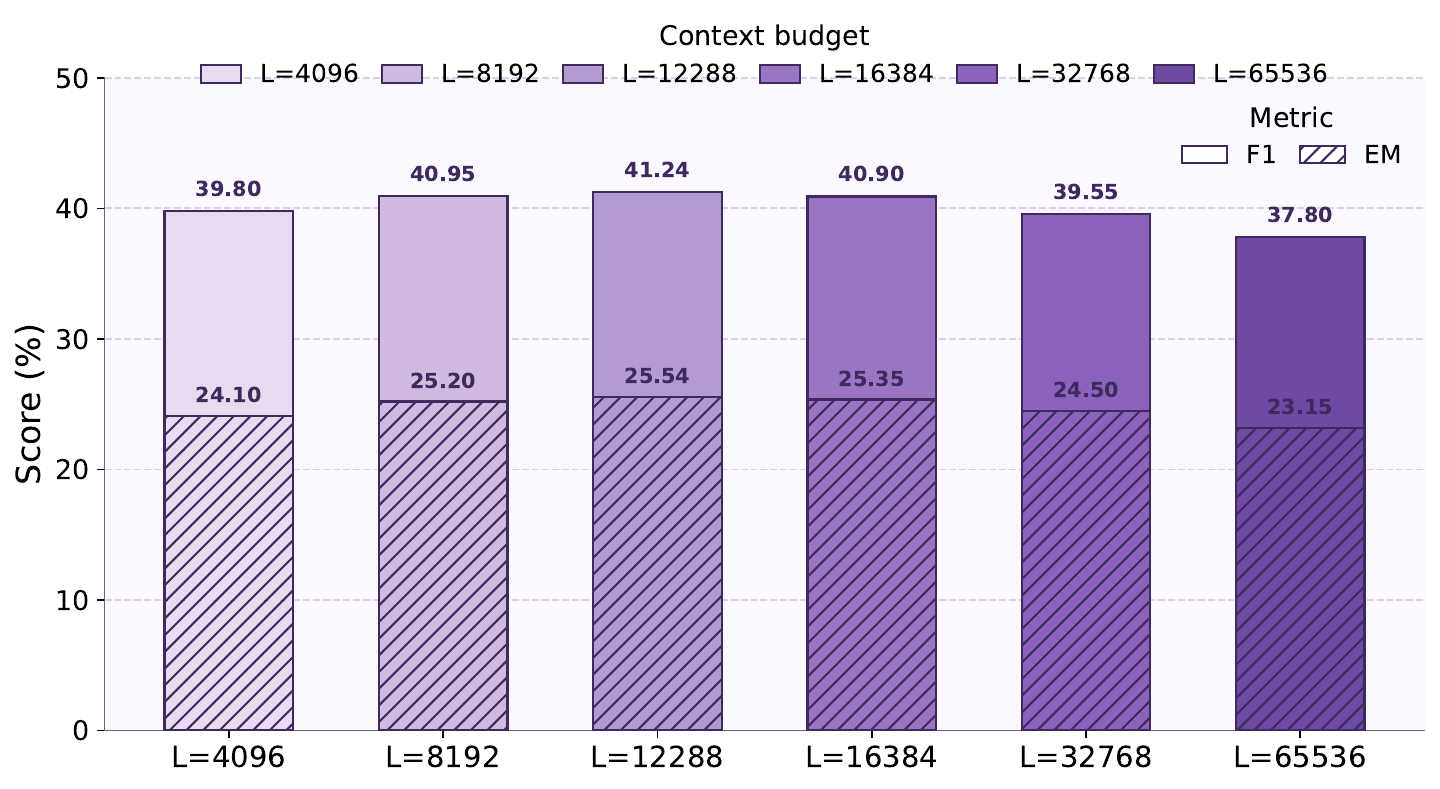}
\caption{LongMemEval-S}
\label{fig:context-budget-longmemeval}
\end{subfigure}
\caption{Effect of context budget $L_{\max}$ on \texttt{TMEM} with Qwen3-4B. Solid bars denote F1 and hatched bars denote EM.}
\label{fig:context-budget-ablation}
\end{figure}

Both datasets exhibit a consistent bell-curve pattern: performance rises as $L_{\max}$ increases toward the optimal value and falls on either side. An excessively small $L_{\max}$ triggers memory extraction too frequently, causing aggressive compression that can discard fine-grained details before they are absorbed into $\Delta_t$. An excessively large $L_{\max}$ allows the working context to grow unwieldy, making the extraction prompt harder to ground and reducing the quality of the QA pairs written into parametric memory. The optimal budgets differ across benchmarks---$L_{\max}{=}4096$ for LoCoMo and $L_{\max}{=}12288$ for LongMemEval-S---which is consistent with LongMemEval-S requiring a larger working window before each TTT update.

\paragraph{SVD initialization vs.\ random initialization.}
Table~\ref{tab:ablation-init} compares SVD-based subspace initialization against standard random initialization on both benchmarks.

\begin{table}[h]
\centering
\caption{Ablation on initialization strategy (Qwen3-4B) across LoCoMo and LongMemEval-S.}
\label{tab:ablation-init}
\small
\begin{tabular}{lcccc}
\toprule
\multirow{2}{*}{Initialization} & \multicolumn{2}{c}{LoCoMo} & \multicolumn{2}{c}{LongMemEval-S} \\
\cmidrule(lr){2-3} \cmidrule(lr){4-5}
 & F1 (\%) & EM (\%) & F1 (\%) & EM (\%) \\
\midrule
Random Gaussian & 24.26 & 12.15 & 38.54 & 22.33 \\
SVD (ours)      & 25.72 & 15.40 & 41.24 & 25.54 \\
\bottomrule
\end{tabular}
\end{table}

SVD initialization consistently outperforms random Gaussian initialization on both benchmarks. The gain is most pronounced on EM (+3.25 on LoCoMo, +3.21 on LongMemEval-S), suggesting that anchoring the LoRA subspace to the principal directions of the pretrained weights improves precise recall rather than just broader coverage. This aligns with the theoretical prediction of Theorem~\ref{thm:svd}: by pre-selecting the row space that captures the largest singular values of $W$, online TTT updates only need to learn the coefficient matrix $B$, which is useful in the few-gradient-step regime available at each trigger.

\paragraph{Supervision form for $\mathcal{T}$.}
Table~\ref{tab:ablation-supervision} studies how the supervision signal used by $\mathcal{T}$ affects downstream memory quality across both benchmarks.

\begin{table}[h]
\centering
\caption{Ablation on supervision form (Qwen3-4B) across LoCoMo and LongMemEval-S.}
\label{tab:ablation-supervision}
\small
\begin{tabular}{lcccc}
\toprule
\multirow{2}{*}{Supervision signal} & \multicolumn{2}{c}{LoCoMo} & \multicolumn{2}{c}{LongMemEval-S} \\
\cmidrule(lr){2-3} \cmidrule(lr){4-5}
 & F1 (\%) & EM (\%) & F1 (\%) & EM (\%) \\
\midrule
Raw next-token prediction & 21.19 & 10.74 & 10.37 & 7.62 \\
Free-form summary         & 24.86 & 14.28 & 35.44 & 20.18 \\
QA pairs (default)        & 25.72 & 15.40 & 41.24 & 25.54 \\
\bottomrule
\end{tabular}
\end{table}

The choice of supervision signal has a substantial effect, particularly on LongMemEval-S. Raw next-token prediction performs the worst by a large margin (10.37 F1 on LongMemEval-S), as it trains the model to reproduce the raw context verbatim rather than distill task-relevant knowledge. Free-form summary improves over this baseline by providing higher-level abstractions, but still falls short of QA pairs, which explicitly encode question--answer structure aligned with downstream memory queries. The gap between free-form summary and QA pairs is especially large on LongMemEval-S (+5.80 F1, +5.36 EM), where many questions require precise factual recall over long dialogue histories.

\paragraph{Freezing the SVD subspace.}
\begin{table}[h]
\centering
\caption{Ablation on freezing SVD subspace (Qwen3-4B) across LoCoMo and LongMemEval-S.}
\label{tab:ablation-freezeA}
\small
\begin{tabular}{lcccc}
\toprule
\multirow{2}{*}{Update mode} & \multicolumn{2}{c}{LoCoMo} & \multicolumn{2}{c}{LongMemEval-S} \\
\cmidrule(lr){2-3} \cmidrule(lr){4-5}
 & F1 (\%) & EM (\%) & F1 (\%) & EM (\%) \\
\midrule
Train $B$ only (freeze $A$) & 25.72 & 15.40 & 41.24 & 25.54 \\
Train $A$+$B$               & 25.89 & 15.58 & 41.51 & 25.66 \\
\bottomrule
\end{tabular}
\end{table}

Table~\ref{tab:ablation-freezeA} shows that allowing both $A$ and $B$ to be updated yields only marginal improvements over freezing $A$ (+0.17 F1 / +0.18 EM on LoCoMo; +0.27 F1 / +0.12 EM on LongMemEval-S). This is consistent with the SVD-initialized subspace already covering most useful few-step adaptation directions in these benchmarks. Freezing $A$ reduces the number of online-updated LoRA factors and preserves nearly the same accuracy, making it the preferred default in our experiments.

\subsection{Efficiency comparison with memory strategies}

Table~\ref{tab:efficiency-combined} compares the computational and memory overhead of \texttt{TMEM} against existing memory strategies: no memory, retrieval-based (A-MEM), and summary-based (MemAgent), across both LoCoMo and LongMemEval-S. Each row uses Qwen3-4B under the same local evaluation pipeline and batch size one. Wall-clock time is measured per episode and includes memory operations performed during the episode as well as SVD adapter initialization on the first episode; GPU memory is the peak allocated memory during evaluation.

\begin{table}[h]
\centering
\caption{Efficiency comparison across memory strategies (Qwen3-4B) on LoCoMo and LongMemEval-S. Wall-clock time and GPU memory are measured per episode under the same evaluation pipeline.}
\label{tab:efficiency-combined}
\small
\begin{tabular}{lcccc}
\toprule
\multirow{2}{*}{Method} & \multicolumn{2}{c}{LoCoMo} & \multicolumn{2}{c}{LongMemEval-S} \\
\cmidrule(lr){2-3} \cmidrule(lr){4-5}
 & Time (s/ep) & GPU Mem (GB) & Time (s/ep) & GPU Mem (GB) \\
\midrule
No memory            & 1.69 & 46.3 & 8.73 & 78.5 \\
Retrieval (A-MEM)    & 2.03 & 13.4 & 15.33 & 13.2 \\
Summary (MemAgent)   & 0.94 & 11.2 & 5.12 & 12.7 \\
\texttt{TMEM} (ours) & 1.53 & 21.3 & 6.55 & 22.9 \\
\bottomrule
\end{tabular}
\end{table}

Several observations stand out. First, the no-memory baseline incurs the highest GPU memory cost (46.3 GB on LoCoMo, 78.5 GB on LongMemEval-S) because the raw context grows throughout the episode. Second, retrieval-based A-MEM is the slowest on LongMemEval-S (15.33 s/ep), consistent with the cost of maintaining and querying an external memory store over a large history. Third, summary-based MemAgent achieves the lowest wall-clock time and GPU footprint but also lower accuracy than \texttt{TMEM} in the main tables. \texttt{TMEM} sits between retrieval and summary in both time and GPU cost: the per-trigger online LoRA update adds overhead relative to summary memory, while still remaining faster than retrieval on LongMemEval-S and substantially lighter than the raw-context no-memory baseline.

\section{Discussion}

\paragraph{Relation to prior memory agents.}
The fast-weight rollout view separates working context $h_t$, textual memory $m_t$, and fast weights $\Delta_t$. Existing context-management agents are recovered by keeping $\Delta_t\equiv0$ and using only the textual update $g(m_t,a_t)$, so the policy remains a frozen model conditioned on rewritten or retrieved prompt-space memory. \texttt{TMEM} changes the rollout by using the memory prompt $d$ to make the boundary action $a_{t_i}$ into QA-style supervision and then applying $\mathcal{T}$ to update $\Delta_{t_i+1}$; subsequent actions are therefore sampled from $\pi_{\theta_0+\Delta_{t_i+1}}$, not only from a frozen model conditioned on compressed text.

\paragraph{Practical trade-off.}
The main control knobs are the trigger budget $L_{\max}$ and the update operator $\mathcal{T}$. A small $L_{\max}$ triggers frequent extraction and online updates but risks over-compressing local evidence; a large $L_{\max}$ reduces update cost but makes the boundary extraction harder and the prompt more expensive. The role of $\mathcal{T}$ is to decide how strongly the extracted supervision is written into $\Delta_t$. Thus, \texttt{TMEM} trades modest online-training overhead for a memory channel that can influence later reasoning without repeatedly placing all past evidence back into the prompt.

\section{Conclusion}

We presented \texttt{TMEM}, a self-evolving parametric memory framework that treats test-time LoRA updates as part of the agent's rollout dynamics rather than as an external post-processing step. By jointly modeling working context $h_t$, explicit memory $m_t$, and fast weights $\Delta_t$, the framework unifies prompt-space memory methods with a parametric memory channel that can directly alter future decisions within an episode. This perspective leads to a stop-gradient policy optimization objective in which outcome rewards improve both ordinary task actions and memory-writing actions, so the base model learns to generate supervision that is useful for subsequent online adaptation. Across conversational memory, search-agent memory, and context-learning benchmarks, \texttt{TMEM} consistently outperforms summary-based and retrieval-based baselines while maintaining practical efficiency. Overall, the results support a simple claim: for long-horizon agents, memory is most effective when it can both be read from context and written into fast model parameters at test time.

\newpage

\bibliography{biblio}
\bibliographystyle{colm2024_conference}

\appendix
\newpage

\section{Approximation Advantage of the SVD Row Space}\label{app:svd}

We formalize the claim that initializing the LoRA projection matrix via the truncated SVD of the pretrained weight yields no larger an approximation error than random Gaussian initialization (strictly smaller whenever $\rho_r > r/d_{\text{in}}$), under a mild spectral alignment condition.

\paragraph{Setup.}
Let $W \in \mathbb{R}^{d_{\text{out}} \times d_{\text{in}}}$ be a pretrained weight matrix with SVD $W = U \Sigma V^{\top}$. Write $V_r = [v_1, \dots, v_r] \in \mathbb{R}^{d_{\text{in}} \times r}$ for the top-$r$ right singular vectors. We assume $r \leq \operatorname{rank}(W)$ (so that the top-$r$ singular values are strictly positive) and that the optimal full-rank weight perturbation $\Delta^{*} \in \mathbb{R}^{d_{\text{out}} \times d_{\text{in}}}$ for the downstream task satisfies $\Delta^{*} \neq 0$. Under the LoRA parameterization $\Delta = BA$ with a fixed projection $A \in \mathbb{R}^{r \times d_{\text{in}}}$ ($\operatorname{rank}(A) = r$) and a learnable coefficient matrix $B \in \mathbb{R}^{d_{\text{out}} \times r}$, the best achievable approximation error is
\[
\mathcal{E}(A) \;=\; \min_{B \in \mathbb{R}^{d_{\text{out}} \times r}} \|\Delta^{*} - BA\|_F^2.
\]
Define the \emph{spectral alignment coefficient}
\[
\rho_r(\Delta^{*}, W) \;=\; \frac{\|\Delta^{*} V_r\|_F^2}{\|\Delta^{*}\|_F^2} \;\in\; [0,\,1],
\]
which measures the fraction of energy of $\Delta^{*}$ that falls into the top-$r$ right singular subspace of $W$.

\begin{assumption}[Spectral alignment]\label{asm:align}
The optimal downstream perturbation satisfies $\rho_r(\Delta^{*}, W) \;\geq\; \dfrac{r}{d_{\textup{in}}}$.
\end{assumption}

\noindent Assumption~\ref{asm:align} requires only that $\Delta^{*}$ concentrates at least as much energy on the top-$r$ subspace of $W$ as would a uniformly random $r$-dimensional subspace in expectation. Since pretrained weight spectra decay rapidly (often as a power law $\sigma_j \propto j^{-\alpha}$) and fine-tuning updates empirically concentrate on the leading singular directions \citep{aghajanyan2021intrinsic}, the condition $\rho_r \gg r/d_{\text{in}}$ holds comfortably in practice.

\paragraph{Formal statement of Theorem~\ref{thm:svd}.}
Let $A_{\textup{SVD}} = \Sigma_r\,V_r^{\top}$ be the SVD initialization derived from $W$, and let $A_{\textup{rand}} \in \mathbb{R}^{r \times d_{\textup{in}}}$ be a random initialization whose rows are drawn i.i.d.\ from $\mathcal{N}(0, \sigma^2 I_{d_{\textup{in}}})$. Under Assumption~\ref{asm:align},
\[
\mathcal{E}(A_{\textup{SVD}}) \;\leq\; \E\bigl[\mathcal{E}(A_{\textup{rand}})\bigr],
\]
with equality if and only if $\rho_r(\Delta^{*}, W) = r / d_{\textup{in}}$.

\begin{proof}
\textbf{Reduction to row-space projection.}
For fixed full-rank $A$, the least-squares optimum is $B^{*} = \Delta^{*} A^{\top}(AA^{\top})^{-1}$, giving residual
\[
\Delta^{*} - B^{*}A = \Delta^{*}\bigl(I - A^{\top}(AA^{\top})^{-1}A\bigr) = \Delta^{*}\,P_{\operatorname{row}(A)}^{\perp},
\]
where $\operatorname{row}(A) \subseteq \mathbb{R}^{d_{\text{in}}}$ denotes the row space of $A$ (the span of its $r$ row vectors),
\[
P_{\operatorname{row}(A)} \;=\; A^{\top}(AA^{\top})^{-1}A
\]
is the orthogonal projection onto $\operatorname{row}(A)$, and $P_{\operatorname{row}(A)}^{\perp} = I - P_{\operatorname{row}(A)}$ projects onto its orthogonal complement. Since $P_{\operatorname{row}(A)}$ and $P_{\operatorname{row}(A)}^{\perp}$ are complementary orthogonal projections, the Pythagorean theorem for the Frobenius norm gives
\begin{equation}\label{eq:approx-proj}
\mathcal{E}(A) = \|\Delta^{*}\,P_{\operatorname{row}(A)}^{\perp}\|_F^2 = \|\Delta^{*}\|_F^2 - \|\Delta^{*}\,P_{\operatorname{row}(A)}\|_F^2.
\end{equation}
Note that $\mathcal{E}(A)$ depends on $A$ only through its row space.

\medskip\noindent\textbf{Approximation error under SVD initialization.}
Substituting $A_{\text{SVD}} = \Sigma_r\,V_r^{\top}$ into the projection formula, and using $V_r^{\top}V_r = I_r$ (orthonormality of the right singular vectors):
\[
P_{\operatorname{row}(A_{\text{SVD}})}
= V_r\,\Sigma_r\,\bigl(\Sigma_r\,V_r^{\top}\,V_r\,\Sigma_r\bigr)^{-1}\,\Sigma_r\,V_r^{\top}
= V_r\,\Sigma_r\,\Sigma_r^{-2}\,\Sigma_r\,V_r^{\top}
= V_r V_r^{\top}.
\]
Substituting into~\eqref{eq:approx-proj}:
\[
\|\Delta^{*} V_r V_r^{\top}\|_F^2 = \operatorname{tr}\!\bigl(V_r^{\top} \Delta^{*\top}\!\Delta^{*} V_r\bigr) = \sum_{j=1}^{r} \|\Delta^{*} v_j\|^2 = \rho_r\,\|\Delta^{*}\|_F^2.
\]
Hence
\begin{equation}\label{eq:error-svd}
\mathcal{E}(A_{\text{SVD}}) = (1 - \rho_r)\|\Delta^{*}\|_F^2.
\end{equation}

\medskip\noindent\textbf{Approximation error under Gaussian random initialization.}
Substituting $A_{\text{rand}}$ into the projection formula gives
\[
P_{\operatorname{row}(A_{\text{rand}})} = A_{\text{rand}}^{\top}(A_{\text{rand}}\,A_{\text{rand}}^{\top})^{-1}A_{\text{rand}}.
\]
To simplify, let $A_{\text{rand}}^{\top} = QR$ be the thin QR decomposition, where $Q \in \mathbb{R}^{d_{\text{in}} \times r}$ has orthonormal columns and $R \in \mathbb{R}^{r \times r}$ is invertible. Then $A_{\text{rand}} = R^{\top}Q^{\top}$, so $A_{\text{rand}}\,A_{\text{rand}}^{\top} = R^{\top}R$ and
\[
P_{\operatorname{row}(A_{\text{rand}})} = QR\,(R^{\top}R)^{-1}\,R^{\top}Q^{\top} = Q\,Q^{\top}.
\]
When the rows of $A_{\text{rand}}$ are i.i.d.\ $\mathcal{N}(0, \sigma^2 I)$, the column span of $Q$ (i.e.\ the row space of $A_{\text{rand}}$) is distributed according to the Haar measure on the Grassmannian $\operatorname{Gr}(r, d_{\text{in}})$, by rotational invariance of the Gaussian distribution. By the symmetry of the Haar measure,
\[
\E[QQ^{\top}] = \frac{r}{d_{\text{in}}}\,I_{d_{\text{in}}}.
\]
Taking the expectation of the projected energy in~\eqref{eq:approx-proj}:
\[
\E\bigl[\|\Delta^{*} QQ^{\top}\|_F^2\bigr]
= \E\bigl[\operatorname{tr}\!\bigl(QQ^{\top} \Delta^{*\top}\!\Delta^{*}\, QQ^{\top}\bigr)\bigr]
= \E\bigl[\operatorname{tr}\!\bigl(\Delta^{*\top}\!\Delta^{*}\, QQ^{\top}\bigr)\bigr]
= \operatorname{tr}\!\bigl(\Delta^{*\top}\!\Delta^{*}\cdot \E[QQ^{\top}]\bigr)
= \frac{r}{d_{\text{in}}}\,\|\Delta^{*}\|_F^2,
\]
where the second equality uses the idempotency $(QQ^{\top})^2 = QQ^{\top}$ and the cyclic property of trace, and the third exchanges $\operatorname{tr}$ and $\E$ by linearity.
Hence
\begin{equation}\label{eq:error-rand}
\E[\mathcal{E}(A_{\text{rand}})] = \bigl(1 - r/d_{\text{in}}\bigr)\|\Delta^{*}\|_F^2.
\end{equation}

\medskip\noindent\textbf{Comparing the two initializations.}
Subtracting~\eqref{eq:error-svd} from~\eqref{eq:error-rand}:
\[
\E[\mathcal{E}(A_{\text{rand}})] - \mathcal{E}(A_{\text{SVD}}) = \Bigl(\rho_r - \frac{r}{d_{\text{in}}}\Bigr)\|\Delta^{*}\|_F^2 \;\geq\; 0,
\]
where the inequality follows from Assumption~\ref{asm:align}. Equality holds if and only if $\rho_r = r/d_{\text{in}}$.
\end{proof}

\begin{remark}[Implicit preconditioning effect]\label{rmk:precond}
Beyond the approximation gap, SVD initialization also improves the optimization dynamics of learning $B$. Each gradient step produces the effective weight update $\Delta_{k+1} - \Delta_k = -\eta\,G_k\,A^{\top}\!A$, where $G_k = \nabla_{\Delta}\ell_k$. For SVD initialization, $A^{\top}\!A = V_r \Sigma_r^2 V_r^{\top}$, so the effective step size along the $j$-th singular direction is $\eta\sigma_j^2(W)$: more important directions receive proportionally larger updates, acting as an adaptive preconditioner. Random initialization instead yields a rank-$r$ matrix $A^{\top}\!A$ whose nonzero eigenvalues concentrate near $\sigma^2 r$, so within the random row space all directions are weighted equally---there is no preferential alignment with the pretrained spectrum. In the TTT regime where only a few gradient steps are taken per trigger, the directional bias of SVD initialization substantially accelerates convergence.
\end{remark}

\newpage

\section{Derivation of Stop-Gradient Policy Optimization}\label{app:pg}

Let $\tau$ denote a complete fast-weight rollout. The stochastic components of the rollout are actions $a_t$, environment observations, and the randomness of the online update operator $\mathcal{T}$. Let $\{t_i\}_{i=1}^{N}$ be the memory-trigger indices, with $t_0=0$ and $t_{N+1}=T$. Define
\[
\bar c_i =
\begin{cases}
(q,h_{t_i},m_{t_i},d), & 1\leq i\leq N,\\
(q,h_T,m_T), & i=N+1.
\end{cases}
\]
Under the stop-gradient convention, the distribution of $\Delta_t$ produced by $\mathcal{T}$ is treated as fixed when differentiating with respect to $\theta_0$. The trajectory likelihood can therefore be written, up to terms independent of $\theta_0$, as
\[
p(\tau\mid\theta_0) \propto
\prod_{i=1}^{N+1}
\left[
\prod_{t=t_{i-1}+1}^{t_i-1}
\pi_{\theta_0+\operatorname{sg}(\Delta_t)}(a_t\mid q,h_t,m_t)
\cdot
\pi_{\theta_0+\operatorname{sg}(\Delta_{t_i})}(a_{t_i}\mid \bar c_i)
\right]
\]
where $d$ is the extraction prompt used to elicit QA-pair supervision. This is the fast-weight analogue of a segment-wise rollout decomposition: each segment contains ordinary rollout actions followed by a boundary action, which is a memory-writing action for $i\leq N$ and the final task response for $i=N+1$.

For the outcome-reward objective
\[
J(\theta_0)=\mathbb{E}_{\tau\sim p_{\theta_0,\mathcal{T},\mathbb{P}_{env}}}\left[R(\tau)\right],
\]
the log-derivative trick gives
\begin{align}
\nabla_{\theta_0}J(\theta_0)
&=\mathbb{E}_{\tau}\left[R(\tau)\nabla_{\theta_0}\log p(\tau\mid\theta_0)\right] \\
&\approx
\mathbb{E}_{\tau}\left[
R(\tau)
\left(
\sum_{i=1}^{N+1}
\left[
\sum_{t=t_{i-1}+1}^{t_i-1}
\nabla_{\theta_0}\log\pi_{\theta_0+\operatorname{sg}(\Delta_t)}(a_t\mid q,h_t,m_t)
+
\nabla_{\theta_0}\log\pi_{\theta_0+\operatorname{sg}(\Delta_{t_i})}(a_{t_i}\mid \bar c_i)
\right]
\right)
\right].
\end{align}
The approximation consists exactly of stopping gradients through $\mathcal{T}$ and through the fast-weight values it produces. Thus, gradients are assigned segment by segment to ordinary task behavior, to the QA-pair extraction actions that provide online LoRA supervision, and to the terminal response, while avoiding back-propagation through the online optimization itself.

\section{Unfiltered CL-Bench Accuracy}\label{app:clbench-unfiltered}

For completeness, we report rubric accuracy on the unfiltered CL-Bench split. Following the main text presentation, this appendix table reports only accuracy rates.

\begin{table}[h]
\centering
\caption{Accuracy (\%) on the unfiltered CL-Bench split across \textbf{DK}, \textbf{ED}, \textbf{PT}, and \textbf{RS}. \textbf{Bold} indicates the best result per column.}
\label{tab:clbench-unfiltered}
\small
\begin{tabular}{llccccc}
\toprule
Method & Model & DK & ED & PT & RS & Total \\
\midrule
\multirow{2}{*}{No Memory}
  & Qwen3-4B & \textbf{6.49} & 1.51 & 3.82 & \textbf{4.24} & 4.63 \\
  & Qwen3-8B & \textbf{6.49} & 1.01 & 5.31 & 4.06 & 4.90 \\
\midrule
\multirow{2}{*}{MemAgent}
  & Qwen3-4B & 4.83 & 1.51 & 4.46 & 3.89 & 4.11 \\
  & Qwen3-8B & 6.18 & 1.01 & 5.73 & 4.06 & 4.90 \\
\midrule
\multirow{2}{*}{AMEM}
  & Qwen3-4B & 5.28 & 1.51 & 3.82 & 3.89 & 4.11 \\
  & Qwen3-8B & 5.88 & 1.51 & 4.46 & 3.53 & 4.37 \\
\midrule
\multirow{2}{*}{\texttt{TMEM}}
  & Qwen3-4B & \textbf{6.49} & \textbf{2.51} & 5.10 & 4.06 & 5.00 \\
  & Qwen3-8B & 6.33 & 1.51 & \textbf{6.37} & 4.06 & \textbf{5.16} \\
\bottomrule
\end{tabular}
\end{table}

The unfiltered split sizes are DK=663, ED=199, PT=471, RS=566, and Total=1899.


\newpage
\section{prompts use in the task}

\begin{figure}[h]
\centering
\begin{promptbox}[System prompt for LoCoMo and LongMemEval]{purple_dist}
\small
\begin{Verbatim}[fontsize=\small]
You are given one problem to solve, previous extracted QA pairs and one conversation session. 
Your task is to create high-quality supervised fine-tuning (SFT) QA pairs 
grounded ONLY in this session.

Question:
<question> {prompt} </question>

Previous extracted QA pairs:
<qa_history> {qa_history} </qa_history>

Session:
<session> {chunk} </session>
\end{Verbatim}
\end{promptbox}
\caption{System prompt template used in LoCoMo and LongMemEval for extracting grounded SFT QA pairs from the current session.}
\label{fig:system-prompt-locomo-longmemeval}
\end{figure}

\begin{figure}[h]
\centering
\begin{promptbox}[System prompt for search-task iterative reasoning]{purple_dist}
\small
\begin{Verbatim}[fontsize=\small]
You will answer some complex questions through iterative reasoning,
memory updates, and web searches.

At each step, you can see the question, previous interaction history in <history> 
... <history>, search query in <search> ... </search>, and the returned information in
<information> ... </information> (except the first step where you will be 
given only the question).

Then choose one of the following actions:
   - If any question remains unanswered, issue a single query for one question inside 
   <search> ... </search>. The query should consist of keywords or a short phrase. 
   Only search one question at a time.
   - If all questions are answered, provide the final answers, separated by semicolons, 
   within <answer> answer1; answer2; ... </answer>. The answers must be concise, 
   contain only essential words, and avoid any explanations.

Your output should be either <search> ... </search> or <answer> ... </answer>.

<question> {question} </question>

<history> {history} </history>

<search> {search} </search>

<information> {information} </information>
\end{Verbatim}
\end{promptbox}
\caption{System prompt template used in the search task for iterative reasoning, grounded memory updates, and step-by-step web search.}
\label{fig:system-prompt-search}
\end{figure}

\newpage

\section{Case Study}

\subsection{LoCoMo}

\begin{figure}[h]
\centering
\begin{promptbox}[LoCoMo case example]{purple_dist}
\small
\begin{Verbatim}[fontsize=\small]
Question: What did the charity race raise awareness for?
Model answer: Mental health
Ground truth: mental health
\end{Verbatim}

\begin{Verbatim}[fontsize=\scriptsize]
[
  {
    "instruction": "What is Caroline's current plan to help with mental health?",
    "output": "She is looking into counseling and mental health jobs"
  },
  {
    "instruction": "Why does Melanie run?",
    "output": "To de-stress and clear her mind"
  },
  {
    "instruction": "What activity does Melanie do to improve her mental health?",
    "output": "Running"
  },
  {
    "instruction": "What charity event did Melanie participate in last Saturday?",
    "output": "A charity race for mental health"
  },
  {
    "instruction": "What did Melanie realize about self-care after the charity race?",
    "output": "That self-care is really important and helps her better look after her family"
  },
  {
    "instruction": "What activities does Melanie carve out for herself each day 
    to refresh and stay present?",
    "output": "Running, reading, or playing her violin"
  },
  {
    "instruction": "What career areas is Caroline interested in exploring?",
    "output": "Counseling or working in mental health"
  }
]
\end{Verbatim}  
\end{promptbox}
\caption{A LoCoMo case with retrieved memory QA pairs and the final question. The model answer matches the ground truth.}
\label{fig:locomo-case-mental-health}
\end{figure}

\begin{figure}[h]
\centering
\begin{promptbox}[LoCoMo case example (adoption planning)]{purple_dist}
\small
\begin{Verbatim}[fontsize=\small]
Question: What are Caroline's plans for the summer?
Model answer: researching adoption with an agency that supports LGBTQ+ families
Ground truth: researching adoption agencies
\end{Verbatim}

\begin{Verbatim}[fontsize=\scriptsize]
[ 
  {
    "instruction": "When did Caroline pass the adoption agency interviews?",
    "output": "Last Friday"
  },
  {
    "instruction": "What does Caroline hope to achieve through adoption?",
    "output": "To build a safe and loving home for kids who haven't had a roof over their heads"
  },
  {
    "instruction": "What did Caroline do this week?",
    "output": "Caroline applied to adoption agencies."
  },
  {
    "instruction": "What event did Caroline attend last Friday?",
    "output": "A council meeting for adoption"
  },
  {
    "instruction": "What is Caroline's goal in the adoption process?",
    "output": "To give kids who need it a loving home"
  },
  {
    "instruction": "Why did Caroline choose the adoption agency she is researching?",
    "output": "Because it helps LGBTQ+ folks with adoption, and its inclusivity and support spoke to her"
  },
  {
    "instruction": "What does Caroline feel about the challenges of being a single parent in the adoption process?",
    "output": "She feels it will be tough, but she is up for the challenge"
  }
]
\end{Verbatim}
\end{promptbox}
\caption{A LoCoMo case focused on adoption planning. The prediction is semantically aligned with the ground truth.}
\label{fig:locomo-case-adoption}
\end{figure}

\begin{figure}[h]
\centering
\begin{promptbox}[LoCoMo case example (family camping preference)]{purple_dist}
\small

\begin{Verbatim}[fontsize=\small]
Question: Would Melanie be more interested in going to a national park or a theme park?
Model answer: national park
Ground truth: National park; she likes the outdoors
\end{Verbatim}

\begin{Verbatim}[fontsize=\scriptsize]
[ 
  {
    "instruction": "What did Melanie say she loves most about camping with her family?",
    "output": "It's a chance to be present and together. We bond over stories, campfires and nature."
  },
  {
    "instruction": "What did Melanie say about the importance of family after the accident?",
    "output": "Family's super important to me. They mean the world to me and I'm so thankful to have them."
  },
  {
    "instruction": "What is Melanie's family's summer tradition?",
    "output": "A family camping trip"
  },
  {
    "instruction": "What was Melanie's best camping memory?",
    "output": "Seeing the Perseid meteor shower last year"
  },
  {
    "instruction": "What did Melanie do with her family last week?",
    "output": "She took her family camping in the mountains."
  },
  {
    "instruction": "What activities did Melanie and her family do during their camping trip?",
    "output": "They explored nature, roasted marshmallows around the campfire, and went on a hike."
  }
]
\end{Verbatim}
\end{promptbox}
\caption{A LoCoMo case about family-outdoor preference. The model answer matches the labeled preference.}
\label{fig:locomo-case-camping}
\end{figure}

\clearpage

\subsection{LongMemEval}

\begin{figure}[h]
\centering
\begin{promptbox}[LongMemEval case QA pairs as JSON (part 1)]{purple_dist}
\scriptsize

\begin{Verbatim}[fontsize=\small]
Question: What is the total amount of money I earned from selling my products at the markets?,
Model answer: $225 + $150 + $120 = $495,
Ground truth: $495,
\end{Verbatim}

\begin{Verbatim}[fontsize=\scriptsize]
[
    {
      "instruction": "How much money did I earn from selling my homemade jam at the Homemade
      and Handmade Market on May 29th?",
      "output": "$225"
    },
    {
      "instruction": "What date did I sell 15 jars of my homemade jam at the Homemade and Handmade Market?",
      "output": "May 29th"
    },
    {
      "instruction": "What is the name of the gift basket I decided to use for my products?",
      "output": "Harvest Delight"
    },
    {
      "instruction": "What is a description I used for my 'Harvest Delight' gift basket on my website or social media?",
      "output": "Indulge in the flavors of the harvest season with our 'Harvest Delight' gift basket, 
      filled with artisanal goodies like our signature homemade jam, warm spice blends, and crunchy nuts."
    },
    {
      "instruction": "How much money did I earn from selling 20 potted herb plants at the Summer Solstice Market?",
      "output": "$150"
    },
    {
      "instruction": "What was the price per potted herb plant sold at the Summer Solstice Market?",
      "output": "$7.5"
    },
    {
      "instruction": "What is the name of the market where I sold 20 potted herb plants?",
      "output": "Summer Solstice Market"
    },
    {
      "instruction": "What are some popular herbal tea flavors mentioned in the session?",
      "output": "Peppermint, Chamomile, Hibiscus, Lemon balm, Rosehip, Ginger, Turmeric"
    },
    {
      "instruction": "What are some market trends in the herbal tea industry according to the session?",
      "output": "Wellness and self-care, functional ingredients like adaptogens and probiotics, 
      fruit and floral infusions, caffeine-free options, sustainable and eco-friendly packaging"
    },
    {
      "instruction": "What are some market trends in the spice blend industry according to the session?",
      "output": "Global inspiration, wellness-focused blends, vegan and vegetarian options, artisanal 
      and small-batch blends, DIY and customization"
    },
    {
      "instruction": "What is a recommended approach for highlighting health benefits in marketing for herbal teas?",
      "output": "Focus on functional benefits, avoid medical claims, use credible sources, 
      be transparent about ingredients and sourcing, and emphasize holistic benefits"
    },
    {
      "instruction": "What are some potential rewards in the proposed loyalty program?",
      "output": "Discounts on future purchases, free shipping, free products, early access to new products, 
      special promotions, invitations to events"
    },
    {
      "instruction": "What is the point system for the loyalty program as described in the session?",
      "output": "1 point for every dollar spent, with bonus points of 10 points for every $50 spent"
    },
    {
      "instruction": "What are some ways to promote the loyalty program according to the session?",
      "output": "Email marketing, social media, in-store promotions, a dedicated webpage, early sign-up offers, 
      paid ads, checkout signage, and staff training"
    },
    {
      "instruction": "How much money did I earn from selling 12 bunches of fresh organic herbs at the farmers' 
      market on May 15th?",
      "output": "$120"
    }
]
\end{Verbatim}
\end{promptbox}
\caption{LongMemEval market-earnings case in JSON form (part 1), containing QA pairs 1-15.}
\label{fig:longmemeval-case-market-earnings-part1}
\end{figure}

\begin{figure}[h]
\centering
\begin{promptbox}[LongMemEval case QA pairs as JSON (part 2)]{purple_dist}
\scriptsize
\begin{Verbatim}[fontsize=\scriptsize]
[
  {
    "instruction": "What are some popular herb-based products I could sell at the Harvest Festival Market 
    according to the session?",
    "output": "Herbal teas, herb-infused oils, herb salts, herbal jellies and jams, herbal potpourri, 
    skincare products, spice blends, vinegars, syrups, and gift baskets"
  },
  {
    "instruction": "What are some tips for drying fresh herbs for use in herbal teas and infused oils?",
    "output": "Air drying, using a dehydrator, oven drying around 150F for 1-2 hours, 
    or using a desiccant in an airtight container"
  },
  {
    "instruction": "What are some pricing strategies for herbal teas and infused oils mentioned in the session?",
    "output": "Competitor research, cost-based pricing, value-based pricing, tiered pricing, 
    bundle deals, and seasonal pricing"
  },
  {
    "instruction": "What is a recommended price range for herbal teas and infused oils according to the session?",
    "output": "Herbal teas: $5-15 per ounce; infused oils: $10-30 per ounce"
  },
  {
    "instruction": "What are some herbs suitable for use in herbal skincare products?",
    "output": "Calendula, chamomile, and lavender"
  }
]
\end{Verbatim}
\end{promptbox}
\caption{LongMemEval market-earnings case in JSON form (part 2), containing QA pairs 16-20.}
\label{fig:longmemeval-case-market-earnings-part2}
\end{figure}

\begin{figure}[h]
\centering
\begin{promptbox}[LongMemEval long-dialogue handling experience (4-item JSON)]{purple_dist}
\scriptsize
\begin{Verbatim}[fontsize=\scriptsize]
[
  {
    "instruction": "When answering a long-memory question about total earnings, what should the model do first?",
    "output": "Identify all separate revenue events related to the question before doing any calculation."
  },
  {
    "instruction": "How should the model answer when multiple QA pairs contain separate amounts 
    from different markets?",
    "output": "It should extract each relevant amount, verify that each amount refers to a distinct sale, 
    then sum them."
  },
  {
    "instruction": "Which QA pairs are necessary to answer the market earnings question in this case?",
    "output": "The QA pairs showing $225 from homemade jam, $150 from potted herb plants, 
    and $120 from fresh organic herbs."
  },
  {
    "instruction": "What should the model avoid when using long conversation QA pairs for a numerical answer?",
    "output": "It should avoid using unrelated product, marketing, pricing, or loyalty-program details unless 
    they contribute directly to the requested total."
  }
]
\end{Verbatim}
\end{promptbox}
\caption{Long-dialogue handling experience extracted from this LongMemEval case, shown separately as a 4-item JSON array.}
\label{fig:longmemeval-case-long-dialogue-experience}
\end{figure}

\clearpage

\subsection{Multi-Objective Search}

\begin{figure}[h]
\centering
\begin{promptbox}[Multi-Objective Search QA pairs as JSON (part 1)]{purple_dist}
\scriptsize

\begin{Verbatim}[fontsize=\small]
Question: What color clothing do people of the Netherlands wear during Oranjegekte or to 
celebrate the national holiday Koningsdag?; What was the name of the 1996 loose adaptation of 
William Shakespeare's \"Romeo & Juliet\" written by James Gunn?; What American professional 
Hawaiian surfer born 18 October 1992 won the Rip Curl Pro Portugal?; What nationality were 
social anthropologists Alfred Gell and Edmund Leach
Model answer: orange for the Netherlands; Tromeo and Juliet; John John Florence; British
Ground truth: "orange; Tromeo and Juliet; John John Florence; British"
\end{Verbatim}

\begin{Verbatim}[fontsize=\scriptsize]
[
  {
    "instruction": "What color clothing do people of the Netherlands wear during Oranjegekte or 
    to celebrate Koningsdag?",
    "output": "orange"
  },
  {
    "instruction": "What was the name of the 1996 loose adaptation of William Shakespeare's \"Romeo & Juliet\" 
    written by James Gunn?",
    "output": "Tromeo and Juliet"
  },
  {
    "instruction": "What American professional Hawaiian surfer born 18 October 1992 won the Rip Curl Pro Portugal?",
    "output": "John John Florence"
  },
  {
    "instruction": "What nationality were social anthropologists Alfred Gell and Edmund Leach?",
    "output": "British"
  },
  {
    "instruction": "What is Oranjegekte?",
    "output": "Oranjegekte is a Netherlands phenomenon associated with major sporting events and Koningsdag, 
    often involving orange clothing and decorations."
  },
  {
    "instruction": "Why is orange clothing associated with Koningsdag?",
    "output": "Orange is associated with the Dutch royal family, the House of Orange-Nassau."
  },
  {
    "instruction": "What evidence identifies Tromeo and Juliet as the correct answer?",
    "output": "The evidence says Tromeo and Juliet is a 1996 loose adaptation of Romeo & Juliet and that 
    James Gunn co-wrote the screenplay."
  },
  {
    "instruction": "If a query about a 1996 Romeo and Juliet adaptation returns both Romeo + Juliet and Tromeo and 
    Juliet, how should the model choose?",
    "output": "Prefer the result that matches all constraints, especially James Gunn and loose adaptation; 
    that points to Tromeo and Juliet."
  },
  {
    "instruction": "What search query is effective for identifying the James Gunn Romeo and Juliet adaptation?",
    "output": "1996 loose adaptation Romeo Juliet James Gunn"
  },
  {
    "instruction": "What evidence supports John John Florence as the surfer answer?",
    "output": "The retrieved text says John John Florence was born October 18, 1992, is a Hawaiian professional surfer, 
    and won the Meo Rip Curl Pro in Portugal."
  },
  {
    "instruction": "What evidence supports Alfred Gell and Edmund Leach being British?",
    "output": "The retrieved documents describe Alfred Gell as a British social anthropologist and Edmund Leach as a 
    British social anthropologist."
  },
]
\end{Verbatim}
\end{promptbox}
\caption{The first half of the search-task QA history in original JSON form.}
\label{fig:search-case-json-part1}
\end{figure}

\begin{figure}[t]
\centering
\begin{promptbox}[Multi-Objective Search QA pairs as JSON (part 2)]{purple_dist}
\scriptsize
\begin{Verbatim}[fontsize=\scriptsize]
[
  {
    "instruction": "In a packed multi-question search task, should the model search all remaining questions at once?",
    "output": "No. Search one question at a time so each retrieved passage can be tied to a specific answer."
  },
  {
    "instruction": "What should the model store after each useful search result?",
    "output": "Store a concise QA pair that captures the question, the answer, and only directly supported information."
  },
  {
    "instruction": "How should the model handle near-duplicate retrieved documents?",
    "output": "Use duplicates as confirmation, but avoid adding duplicate QA pairs unless the new document adds a 
    useful constraint."
  },
  {
    "instruction": "How should the final answer be formatted for a packed question with four subquestions?",
    "output": "Return four concise answers in the original question order, separated by semicolons."
  }
]
\end{Verbatim}
\end{promptbox}
\caption{The second half of the search-task QA history in original JSON form.}
\label{fig:search-case-json-part2}
\end{figure}

\clearpage


\end{document}